%% file: main_arxiv.tex
\PassOptionsToPackage{table,dvipsnames}{xcolor}
\documentclass[11pt, a4paper]{causalgame}

\usepackage[numbers,sort&compress,round]{natbib}
\usepackage{mathtools}
\usepackage{bbm}

\usepackage{subcaption}
\usepackage{array}
\usepackage{multirow}
\usepackage{makecell}
\usepackage[normalem]{ulem}
\usepackage{xspace}
\usepackage{pifont}
\usepackage{tikz}
\usepackage[most]{tcolorbox}
\definecolor{pink-color}{RGB}{200,80,120}
\usetikzlibrary{shapes,positioning,arrows.meta,fit,calc,backgrounds,shadows,decorations.pathreplacing,shapes.geometric}

\usepackage{algorithm}
\usepackage{algorithmic}

\theoremstyle{plain}
\newtheorem{theorem}{Theorem}

\newtheorem{lemma}[theorem]{Lemma}
\newtheorem{corollary}[theorem]{Corollary}
\theoremstyle{definition}
\newtheorem{definition}[theorem]{Definition}
\newtheorem{assumption}[theorem]{Assumption}
\theoremstyle{remark}

\newenvironment{myquotation}{\begin{quote}\itshape}{\end{quote}}

\newcommand{\ours}{\texttt{CANA}\xspace}
\newcommand{\oursfull}{\textbf{C}ausal \textbf{Ana}logical Researcher\xspace}
\newcommand{\adrfull}{\text{Analogical Deep Research}\xspace}
\newcommand{\adr}{\texttt{ADR}\xspace}
\newcommand{\sab}{\texttt{SAB}\xspace}

\newcommand{\dr}{\texttt{DR}\xspace}
\newcommand{\adrb}{\texttt{ADR-bench}\xspace}

\newcommand{\chatgpt}{\texttt{OpenAI Deep Research}\xspace}
\newcommand{\gemini}{\texttt{Gemini Deep Research}\xspace}
\newcommand{\qwen}{\texttt{Qwen Deep Research}\xspace}
\newcommand{\miro}{\texttt{MiroFlow}\xspace}

\title{Retrieving and Integrating Historical Analogies for Foresight Analysis}

\author[1,2]{Yongqiang Chen}
\author[1,2]{Guangyi Chen}
\author[1,2]{Yuewen Sun}
\author[1,2]{Kun Zhang}
\affil[1]{Mohamed bin Zayed University of Artificial Intelligence}
\affil[2]{Carnegie Mellon University}

\begin{abstract}
Systematic comparisons between current situations and structurally similar past events in history, i.e., historical analogies, is among the most powerful tools for foresight analysis. In this work, we present a new task called \adrfull (\adr) to Large Language Model (LLM) agents and construct the first ADR benchmark \adrb to study whether LLM agents are able to find and leverage historical analogies when doing foresight analysis.
Our investigation reveals a key obstacle: LLM agents are poor at finding analogies because they match on surface features rather than underlying mechanisms.
We argue that \adr is inherently a causal question as it requires understanding why the event occurred. Based on our theoretical analysis, we propose two principles required for \adr, including the mechanism alignment and cross-analogy confirmation.
Built upon our theoretical results, we propose a new agentic framework called \oursfull (\ours) that guides LLMs to find and integrate historical analogies.
\ours incorporates a simple yet effective structural decomposition representation, and integrates structural feedback for reflective improvements of historical analogy identification and integration.
We show that \ours brings up to 10\% improvements in historical analogy generation, and surpasses the state-of-the-art deep research agents in the \adrb.
Case studies with the ongoing events confirm the effectiveness of \ours in leveraging historical analogies.
\end{abstract}

\begin{document}

\maketitle

\section{Introduction}
\label{sec:intro}
In the study of history, similar events are sometimes found to re-occur and hence systematic comparisons between current situations and similar histories provide us insights for foresight analysis~\citep{achenbaum1983making,guldi2014history,parsons2016historical,ghilani2017looking,HistoricalAnalogies}.
Historical analogies, i.e., structurally similar past events, have been widely used in human history, such as policy makers~\citep{neustadt1986thinking}, scientists~\citep{dunbar1995invivo}, and practitioners~\citep{dalio2017principles}. In particular, \citet{green2007structured} showed that structurally aligned analogies improve expert conflict predictions from 32\% to 60\% accuracy.
Nevertheless, untrained human tend to only come up with superficial analogies instead of structural analysis \citep{ghilani2017looking, khong1992analogies}.

In the meantime, recently Large Language Models (LLMs) has emerged and demonstrated great promise in tackling complex reasoning tasks by learning from the extensive data~\citep{chatgpt,Bubeck2025EarlySA,o1,guo2025deepseek}.
Built upon powerful LLMs, Deep research (\dr) agents that are equipped with more advanced planning capabilities and access to external tools, can further produce in-depth, comprehensive answers to complex research tasks~\citep{openai2025deepresearch}.
Recent sophisticated \dr agents can tackle expert-level question answering problems~\citep{phan2025humanity}, search for the answer with complex constraints~\citep{wei2025browsecomp}, and even do research tasks~\citep{Mitchener2025KosmosAA}.
Therefore, it raises an intriguing research question:
\begin{myquotation}\centering
    \textit{Can \dr agents do analogical search and reasoning for foresight analysis?}
\end{myquotation}
In fact, early attempts with LLMs in finding analogies showed mixed results~\citep{jiayang2023storyanalogy,ye2024analobench,sourati2024arn,li2025past}. LLMs tend to struggle finding analogies from narrative stories~\citep{jiayang2023storyanalogy}, cross-domain narratives~\citep{sourati2024arn}, analogies in the long context~\citep{ye2024analobench}, and exhibit human-like tendencies to superficially similar events~\citep{li2025past}.
In order to answer the question, this work introduces a new task called \adrfull (\adr) to LLM agents, where one aims to find historical events that align with the given target event (or the current situation) in terms of the underlying mechanism.
We construct the first \adr benchmark called \adrb that consists of $15$ events, including $10$ past events with literature-documented historical analogies, as well as $5$ recent events without literature documents.
Through \adrb, we find that the state-of-the-art \dr agents, including \chatgpt~\citep{openai2025deepresearch}, \gemini~\citep{google2024geminideepresearch}, \qwen~\citep{tongyi2025deepresearch}, and \miro~\citep{miromind2026miroflow} do not actively leverage the historical analogies when doing foresight analysis. Even with historical analogies identified, the analysis is rather shallow.

We show that \adr is inherently a causal problem that the agent requires understanding the \textit{underlying hidden mechanism and factors} shared across historical events and the current conditions~\citep{HistoricalAnalogies}.
Aligned with the observations in the literature~\citep{li2025past}, identifying analogies merely based on the surface-level similarity suffers from a non-identifiability issue and degraded foresight prediction.
In addition, sufficient foresight requires sufficient coverage of the underlying structure.

\input{Figures/ADR.tex}

Motivated by our theoretical analysis, we propose \oursfull (\ours), that processes events into structural decompositions including preconditions, temporal chains, mechanisms, and outcomes. Retrieving through the decomposed structural representation of events improves more than 10\% than the state-of-the-art self-reflection method by~\citet{li2025past}. Furthermore, we integrate the structural decomposition into the open-sourced \dr agent \miro, that iteratively retrieves new analogies to cover sufficient aspects of the underlying mechanisms. In \adrb, we show that \ours helps identify more informative analogies, uncovering hidden driving factors, and improves the depth of foresight analysis.
Our contributions can be summarized as follows:
\begin{itemize}[leftmargin=*]
    \item We present the task of \adrfull with the first \adrb, showing the limitation of existing \dr in retrieving and integrating historical analogies for foresight analysis;
    \item We present a formal framework of \adr and show that structural representation and sufficent coverage are two principles to retrieve satisfying historical analogies;
    \item We propose the first \adr agent \ours, that achieves more than 10\% performance in historical analogy retrieval and integration than the state-of-the-art methods.
\end{itemize}

\section{Related Work}
\label{sec:related}
\textbf{Historical analogies for foresight analysis.}
In human history, historical analogies are widely used or cited for foresight~\citep{achenbaum1983making,guldi2014history,parsons2016historical,ghilani2017looking,HistoricalAnalogies}. For example, policymakers often exploit analogies to reason about the unprecedented scenarios~\citep{neustadt1986thinking,houghton1996role,khong1992analogies,brunk2008curing}. Scientists use analogical reasoning to identify useful hypotheses~\citep {dunbar1995invivo,nersessian2008creating}.\citet{green2007structured} also showed that forecasting with analogies can be largely improved from $32\%$ to $60\%$.
Analogies refer to events that share similar mechanisms or aligned structures~\citep{gentner1983structure,falkenhainer1989sme}.
Structural alignment of analogies helps with foresight analysis~\citep{bartha2010parallel}, which is inherently a causal problem where separate partial observations are aligned to recover the underlying causal relations~\citep{huang2020cdmini,adams2021identification,yao2024multiview}. We exploit the conceptual connection between analogy identification and causal learning to build our framework.

\textbf{Analogical reasoning with LLMs.}
Due to the usefulness of analogies, there is a growing body of works trying to use LLMs to do analogical reasoning~\citep{webb2023emergent}. \citet{opielka2025analogical} showed LLMs can abstract concepts and perform analogical reasoning. \citet{yasunaga2024analogical} found that self-generated analogous examples help improve LLM reasoning. \citet{yu2024thought} showed that relating analogical sub-problems helps solve a challenging composite problem with LLMs.
As LLMs are trained on extensive datasets, finding analogies with LLMs is also of great interests~\citep{jiayang2023storyanalogy,yuan2024analogykb,sourati2024arn,li2025past}.
\citet{jiayang2023storyanalogy} found that LLMs are relatively limited in finding analogies from stories.
\citet{sourati2024arn} also showed the limitations of LLMs in finding cross-domain analogies. \citet{ye2024analobench} provided evidence on the limitations of LLMs in finding analogies in the long context.
The most relevant work to ours is \citet{li2025past}, which introduced the historical-analogy benchmark and showed that reflecting on the topics, background, process, and result can improve the historical analogy capabilities of LLMs. Nevertheless, none of the existing works fully define the problem and the principles required for historical analogy identification and integration.

\textbf{Deep research and forecasting with LLMs.}
Deep research (\dr) agent is a systematic framework to equip LLMs with advanced capabilities such as planning and tool use~\citep{Xu2025ACS}, in order to tackle complex real-world problems that require multi-step planning and reasoning~\citep{openai2025deepresearch,li2025webthinker}.
For example, \citet{li2025webthinker,miromind2026miroflow} showed that synthesis of comprehensive search results help tackle challenging reasoning tasks~\citep{wei2025browsecomp,mialon2023gaia,phan2025humanity} and even future prediction~\citep{zeng2026futurex}. In particular, forecasting is inherently a challenging task that requires gather substantive information~\citep{karger2025forecastbench} while of great downstream application value~\citep{yu2023finmem,zhang2024finagent,sen2026llms}.
Despite the importance of historical analogies, we found that \dr agents do not proactively exploit the use of historical analogies by the first \adrb, and propose a new \adr agent to tackle the limitation.

\input{2_ADR.tex}

\input{3_ADR-agent.tex}
\input{4_exp.tex}
\section{Conclusion}
\label{sec:conclusion}
In this work, we introduced Analogical Deep Research (\adr) to LLM agents. We presented a theoretical framework to demonstrate the importance of identifying and aligning historical analogies for foresight analysis. By constructing the first \adr benchmark (\adrb), we showed that existing \dr agents lack the capability of proactively finding and leveraging historical analogies when doing foresight analysis. We derived two principles from theoretical analysis that \adr requires structural alignment and cross-analogy confirmation, and realized the two principles via a framework called \ours.
Our results show that (i) causal-inspired dimensions improve analogy retrieval by $10\%$, (ii) \ours significantly improves the \adr on \adrb.

\bibliography{ref_0_ai_sci,ref_1_llm,ref_2_causality,ref_3_others,ref_adr_related}

\input{9_appdx.tex}

\end{document}

%% file: Figures/ADR.tex
\begin{figure}[!t]
    \centering
    \definecolor{memBlue}{RGB}{220, 235, 255}
    \definecolor{progGreen}{RGB}{220, 255, 220}
    \definecolor{obsGray}{RGB}{240, 240, 240}
    \definecolor{borderGray}{RGB}{100, 100, 100}

\resizebox{\linewidth}{!}{
\begin{tikzpicture}[
    >={Stealth[scale=1.2, inset=1pt]},
    %
    latent node/.style={circle, draw=borderGray, fill=memBlue, thick, minimum size=1.5cm, font=\sffamily\LARGE\bfseries, text=black!80},
    %
    observed node/.style={circle, draw=borderGray, fill=progGreen, thick, minimum size=1.5cm, font=\sffamily\LARGE\bfseries, text=black!80},
    %
    hidden node/.style={circle, draw=borderGray, fill=white, dashed, line width=1.2pt, minimum size=1.5cm, font=\sffamily\Huge\bfseries, text=borderGray},
    %
    missing node/.style={circle, draw=borderGray!50, densely dashed, fill=gray!5, thick, minimum size=1.5cm, font=\sffamily\Huge\bfseries, text=borderGray!60},
    %
    desc box/.style={rectangle, draw=borderGray, fill=obsGray, thick, rounded corners=6pt, minimum height=1.9cm, minimum width=4.6cm, align=left, font=\sffamily\large, text=black!85, drop shadow={opacity=0.12, shadow xshift=1.5mm, shadow yshift=-1.5mm}},
    %
    causal edge/.style={->, thick, color=borderGray, shorten >=2pt, shorten <=2pt},
    mapping edge/.style={->, line width=1.2pt, color=borderGray!60, shorten >=4pt, shorten <=4pt},
    %
    lbl/.style={font=\sffamily\Large\bfseries, align=center, text=borderGray},
    view title/.style={font=\sffamily\LARGE\bfseries, align=right, anchor=east, text=borderGray!100}
]

    \def\xstep{3.1}   
    \def\ystep{-2.6}  
    \def\xbox{6.9}    

    \node[latent node] (L1) at (1*\xstep, 0) {$s_1$};
    \node[latent node] (L2) at (2*\xstep, 0) {$s_2$};
    \node[latent node] (L3) at (3*\xstep, 0) {$s_3$};
    \node[latent node] (L4) at (4*\xstep, 0) {$s_4$};
    \node[latent node] (L5) at (5*\xstep, 0) {$s_5$};

    \draw[causal edge, very thick] (L1) -- (L2);
    \draw[causal edge, very thick] (L2) -- (L3);
    \draw[causal edge, very thick] (L3) -- (L4);
    \draw[causal edge, very thick] (L4) -- (L5);

    \node[lbl, above=0.35cm of L1] {Trigger};
    \node[lbl, above=0.35cm of L2] {Enabler};
    \node[lbl, above=0.35cm of L3] {Amplifier};
    \node[lbl, above=0.35cm of L4] {Mediator};
    \node[lbl, above=0.35cm of L5] {Outcome};
    
    \node[view title, left=0.6cm of L1] {Latent Mechanisms\\ and Structural Roles};

    \begin{scope}[on background layer]
        \foreach \i in {1,...,5} {
            \draw[dotted, borderGray!40, line width=1pt] (\i*\xstep, -0.9) -- (\i*\xstep, 4.2*\ystep);
        }
    \end{scope}


    \node[view title] (T_lbl) at (0.4*\xstep, 1*\ystep) {Target Event $E_T$\\[0.4ex] {\large\normalfont\color{borderGray!90} 2008 GFC}};
    \node[observed node] (T1) at (1*\xstep, 1*\ystep) {$s_1$};
    \node[observed node] (T2) at (2*\xstep, 1*\ystep) {$s_2$};
    \node[hidden node]   (T3) at (3*\xstep, 1*\ystep) {?};
    \node[hidden node]   (T4) at (4*\xstep, 1*\ystep) {?};
    \node[hidden node]   (T5) at (5*\xstep, 1*\ystep) {?};

    \draw[causal edge] (T1) -- (T2);
    \draw[causal edge, dashed] (T2) -- (T3);
    \draw[causal edge, dashed] (T3) -- (T4);
    \draw[causal edge, dashed] (T4) -- (T5);

    \node[view title] (V1_lbl) at (0.4*\xstep, 2*\ystep) {Analogy 1\\[0.4ex] {\large\normalfont\color{borderGray!90} 1907 Panic}};
    \node[observed node] (V1_1) at (1*\xstep, 2*\ystep) {$s_1$};
    \node[observed node] (V1_2) at (2*\xstep, 2*\ystep) {$s_2$};
    \node[observed node] (V1_3) at (3*\xstep, 2*\ystep) {$s_3$};
    \node[missing node]  (V1_4) at (4*\xstep, 2*\ystep) {$\times$};
    \node[missing node]  (V1_5) at (5*\xstep, 2*\ystep) {$\times$};

    \draw[causal edge] (V1_1) -- (V1_2);
    \draw[causal edge] (V1_2) -- (V1_3);

    \node[view title] (V2_lbl) at (0.4*\xstep, 3*\ystep) {Analogy 2\\[0.4ex] {\large\normalfont\color{borderGray!90} Japan 1990}};
    \node[observed node] (V2_1) at (1*\xstep, 3*\ystep) {$s_1$};
    \node[missing node]  (V2_2) at (2*\xstep, 3*\ystep) {$\times$};
    \node[observed node] (V2_3) at (3*\xstep, 3*\ystep) {$s_3$};
    \node[observed node] (V2_4) at (4*\xstep, 3*\ystep) {$s_4$};
    \node[missing node]  (V2_5) at (5*\xstep, 3*\ystep) {$\times$};

    \draw[causal edge] (V2_3) -- (V2_4);

    \node[view title] (V3_lbl) at (0.4*\xstep, 4*\ystep) {Analogy 3\\[0.4ex] {\large\normalfont\color{borderGray!90} LTCM 1998}};
    \node[missing node]  (V3_1) at (1*\xstep, 4*\ystep) {$\times$};
    \node[observed node] (V3_2) at (2*\xstep, 4*\ystep) {$s_2$};
    \node[observed node] (V3_3) at (3*\xstep, 4*\ystep) {$s_3$};
    \node[missing node]  (V3_4) at (4*\xstep, 4*\ystep) {$\times$};
    \node[observed node] (V3_5) at (5*\xstep, 4*\ystep) {$s_5$};

    \draw[causal edge] (V3_2) -- (V3_3);

    \node[view title, font=\sffamily\LARGE\bfseries, align=center] (D_title) at (\xbox*\xstep, 0) {Descriptions};
    
    \node[desc box] (D_T) at (\xbox*\xstep, 1*\ystep) {\textbf{Observed:} $D(E_T)$\\[0.6ex] \normalsize Era: 2008\\ Sector: Housing};
    \node[desc box] (D_1) at (\xbox*\xstep, 2*\ystep) {\textbf{Observed:} $D(E_1)$\\[0.6ex] \normalsize Era: 1907\\ Sector: Copper};
    \node[desc box] (D_2) at (\xbox*\xstep, 3*\ystep) {\textbf{Observed:} $D(E_2)$\\[0.6ex] \normalsize Era: 1990\\ Country: Japan};
    \node[desc box] (D_3) at (\xbox*\xstep, 4*\ystep) {\textbf{Observed:} $D(E_3)$\\[0.6ex] \normalsize Era: 1998\\ Entity: Hedge Fund};

    \begin{scope}[on background layer]
        \node[rectangle, draw=borderGray!60, dashed, thick, rounded corners=12pt, fill=gray!5, inner sep=16pt, fit=(D_title) (D_T) (D_1) (D_2) (D_3)] {};
    \end{scope}

    \draw[mapping edge] (T5.east) -- (D_T.west);
    \draw[mapping edge] (V1_5.east) -- (D_1.west);
    \draw[mapping edge] (V2_5.east) -- (D_2.west);
    \draw[mapping edge] (V3_5.east) -- (D_3.west);

    
    \draw[decorate, decoration={brace, amplitude=10pt}, thick, borderGray] 
        ([xshift=-0.2cm, yshift=-0.2cm]V3_lbl.south west) -- ([xshift=-0.2cm, yshift=0.2cm]V1_lbl.north west) 
        node[midway, left=0.5cm, font=\sffamily\LARGE\bfseries, align=right, text=borderGray] {Historical\\Analogies\\ \Large $\mathcal{A}$};

    \begin{scope}[shift={(3.5, 5.2*\ystep)}]
        \draw[thick, draw=borderGray!40, fill=obsGray!50, rounded corners=8pt] (-0.3, 0.9) rectangle (12.2, -0.9);
        \node[font=\sffamily\Large\bfseries\color{black!80}] at (0.2, 0) [anchor=west] {Legend:};
        
        \node[observed node, minimum size=0.9cm, font=\sffamily\normalsize] at (2.8, 0) {};
        \node[anchor=west, text=black!80, font=\sffamily\large] at (3.3, 0) {Observed};
        
        \node[hidden node, minimum size=0.9cm, font=\sffamily\Large] at (6.3, 0) {?};
        \node[anchor=west, text=black!80, font=\sffamily\large] at (6.8, 0) {Target Hidden};
        
        \node[missing node, minimum size=0.9cm, font=\sffamily\Large] at (10.0, 0) {$\times$};
        \node[anchor=west, text=black!80, font=\sffamily\large] at (10.5, 0) {Missing};
    \end{scope}

\end{tikzpicture}
}
    \caption{Illustration of the \adrfull process. Given the target event $E_T$ as the 2008 Global Financial Crisis (GFC) (target), at the time of August 2007, the target event $E_T$ has two observed positions $s_1$: subprime mortgage defaults (trigger), while $s_3$--$s_5$ are hidden. One could find useful historical analogies where each provides a different partial view of the same pattern. The \textit{1907 Panic} shows a trigger (a failed speculation), a vulnerability (unregulated institutions), and an amplifier (shared directors spreading panic across banks). \textit{Japan 1990} shows a trigger (interest rate hike), an amplifier (cross-ownership ties hiding losses), and a transmission channel (banks kept lending to insolvent borrowers for a decade). \textit{LTCM 1998} shows a vulnerability (an unregulated hedge fund), an amplifier (opaque bilateral contracts), and an outcome (a coordinated bailout).
    No single analogy covers all hidden positions, but their union does. In particular, $s_3$ (amplifier) appears in all three analogies as interlocks, keiretsu, and OTC exposure, establishing it as a cross-analogy invariant and enabling inference that a hidden amplifier (the CDS web) likely exists in the target.
    However, the event descriptions may not share high similarity, and it is challenging for \dr agents to find them.}
    \label{fig:ADR_graph}
\end{figure}
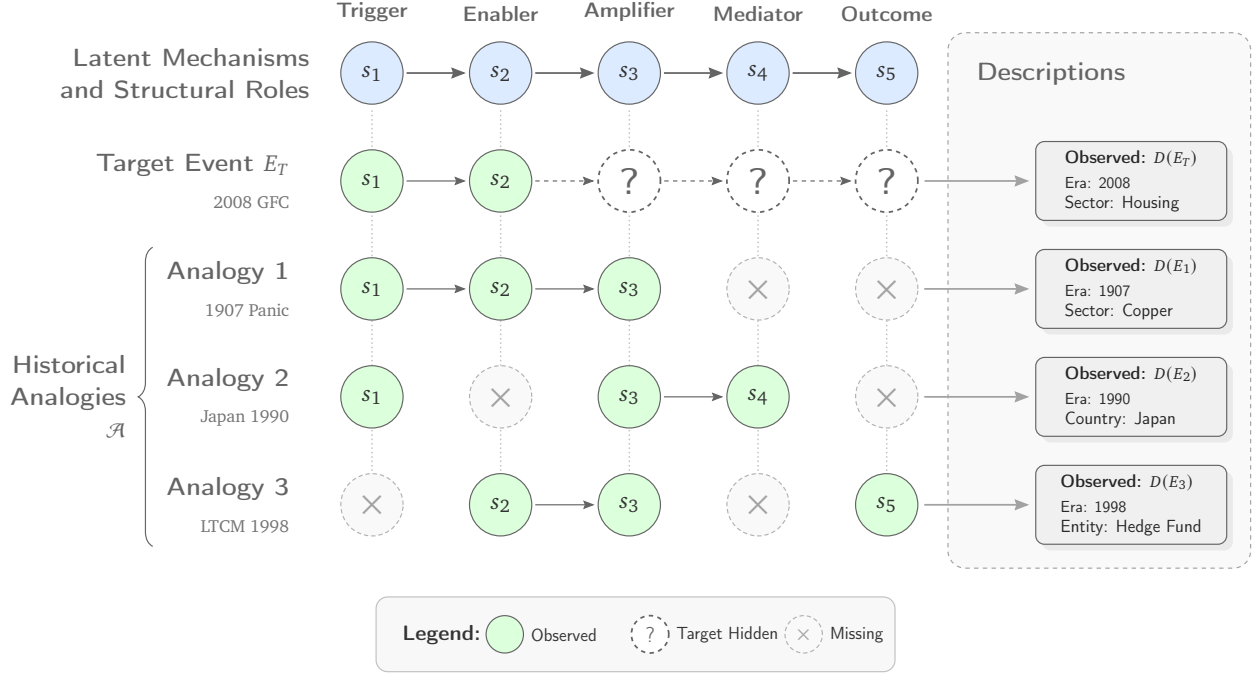

%% file: 2_ADR.tex
\section{Foresight with Historical Analogies}
\label{sec:theory}
In this section, we formally define the problem of \adrfull, along with its objectives and principles to realize \adr.

\subsection{Problem Definition}
\label{sec:definition}
We begin by establishing necessary concepts, and full details of the theoretical framework can be found in Appendix~\ref{appdx:theory}.
\begin{definition}[Event Representation]
\label{def:event}
An event $E$ has two representations: (i)~\textbf{descriptive} $D(E) \in \mathbb{R}^p$, capturing observable surface features (entities, domain, timeline); and (ii)~\textbf{mechanistic} $M(E) = (V_E, \mathcal{E}_E)$, where $V_E = \{v_1, \ldots, v_k\}$ are causal factors, $G_E = (V_E, \mathcal{E}_E)$ is a directed graph encoding causal relationships. Depending on the event-specific context $\chi_E$, the same $M(E)$ can have different instantiations at the descriptive level $D(E) = h(M(E), \chi_E)$.
\end{definition}
Given a target event $E_T$, we aim to find an aligned source event $E_S$ where they share a set of structural positions $\{s_1, \ldots, s_L\}$ in the underlying causal graphs describing the mechanisms of $E_S$ and $E_T$. Positions are not pre-specified but discovered through cross-event alignment~\citep{gentner1983structure}. Shown as in Fig.~\ref{fig:ADR_graph}, roles can be assigned to the objects depending on the relational structure in the event~\citep{clem1991systematicity}.

\textbf{Partial observability.} Usually, the target event describes the current situation, and one can only observe part of the underlying objects $V^{\mathrm{obs}} \subseteq V_{E_T}$ as well as the structural positions  $\mathcal{S}_T^{\mathrm{obs}} = \psi_{E_T}(V^{\mathrm{obs}})$  until the current time, or cutoff time $t_c$, where $\psi_{E_T}(\cdot)$ is a structural mapping that maps underlying objects in $V_{E_T}$ to the corresponding structural roles.
In addition, the remaining objects are hidden $\mathcal{V}_T^{\mathrm{hid}} = V_{E_T} \setminus V^{\mathrm{obs}}$ occupying the hidden positions $\mathcal{S}_T^{\mathrm{hid}} = \psi_{E_T}(V_{E_T}) \setminus \mathcal{S}_T^{\mathrm{obs}}$.

\textbf{\adrfull.} \adr for foresight analysis requires one to find sufficient analogies in the history and align the underlying structural positions to make predictions about future trajectories.

\begin{definition}[Foresight Risk with Analogies]
\label{def:surface}
Let $Y_F = \Gamma(\{\tau_s(t) : s \in \mathcal{S}_T, t > t_c\})$  denote the foresight target, where $\tau_s(t)$ describes the future trajectories of $s$ at time $t$, 
let $P_F^w$ denote the true distribution and $\widehat{P}_F^{\mathcal{A}}$ denote the forecasted distribution through the analogy set $\mathcal{A}$. Then the foresight risk is
$
\mathcal{R}(\mathcal{A}, w) = \sum_{s \in \mathcal{S}_T} \mu_s \cdot \mathrm{TV}(\widehat{P}_s^{\mathcal{A}}, P_s^w)
$
where $\mathrm{TV}(\cdot,\cdot)$ measures the distance between two input distributions, $\mu_s \geq 0$, $\sum_s \mu_s = 1$, and $P_s^w$ is the true trajectory distribution at position $s$.
\end{definition}
Intuitively, let $\mathcal{U}$ denote the universe of events, the ADR problem is to select $\mathcal{A}^* \subseteq \mathcal{U}$ and produce foresight report consisting of uncovered hidden factors, mechanisms, predicted trajectories, conditional scenarios, and key uncertainties. The success of \adr relies on the following assumption.
\begin{assumption}[Mechanism Transfer]
\label{ass:transfer}
If factors $v_S \in V_{E_S}$ and $v_T \in V_{E_T}$ occupy the same structural position ($\psi_{E_S}(v_S) = \psi_{E_T}(v_T) = s$), both are active at their respective cutoffs, and $E_S$ has progressed beyond $E_T$'s current stage, then
$\mathrm{TV}(\widehat{P}_s^{E_S}, P_s^T) \leq \alpha_s$
for some $\alpha_s > 0$.
\end{assumption}
Assumption~\ref{ass:transfer} is at the core of the literature of applied history, where people widely use historical analogies for foresight analysis~\citep{HistoricalAnalogies}.

\subsection{\adrfull Benchmark}
\label{sec:benchmark}
After defining the goal of \adr, we can explore whether the state-of-the-art \dr agents are able to find useful historical analogies for foresight analysis.
We begin by constructing the first \adr benchmark \adrb with the aim of providing insights on the performance of existing \dr agents.

\textbf{Event corpus.} Essentially, we need events where both historical analogies are available and the foresight quality can be verified.
However, one can hardly find events that jointly satisfy all the criteria. Hence, we incorporate two types of events: 10 \textbf{historical events} that are widely studied in the literature, and $5$ \textbf{forward events} covering currently unfolding situations. The $15$ events span financial ($5$), geopolitical ($6$), and technology ($4$) domains.

\textbf{Per-event annotation and task design.}
Each event is annotated with:
(1) temporal cutoff with rationale;
(2) pre-cutoff analyst's brief;
(3) causal structure: including preconditions, causal chain, mechanism, outcome, and structural pattern. For forward events, we only have the approximated mechanism and current outcome since they are unfolding);
(4) oracle analogies with consensus levels and documented shared mechanisms for historical events;
(5) cross-analogy invariants with mechanisms that recur across the oracle analogies; and
(6) hidden factors with foreseeability.
%
We focus on foresight analysis: given only the pre-cutoff brief, produce a forward-looking analysis. The task asks the agent to ``assess the underlying dynamics, evaluate the risks, and provide your forward-looking view,'' without mentioning analogies, historical parallels, or hidden factors.
If an agent uses analogies, it is because its architecture drives it to do so.
 
\textbf{Evaluation design.}
We use a rubric-based evaluation protocol. Detailed instructions for evaluation can be found in Appendix~\ref{appdx:adrb}.
Specifically, we will extract 5-level claims, including (L1) surface fact, (L2) single-analogy insight, (L3-D) cross-analogy pattern without structural role, (L3-S) cross-analogy pattern with causal role specified and L4 hidden factor inferred from cross-analogy evidence from $\geq$2 events. We also report the hit rate of hidden factors.
We also score analysis quality (RAS), the depth of foresight analysis (FQS), and the mechanism grounding quality (Gr).

\begin{figure}[ht]
    \centering
    \includegraphics[width=\linewidth]{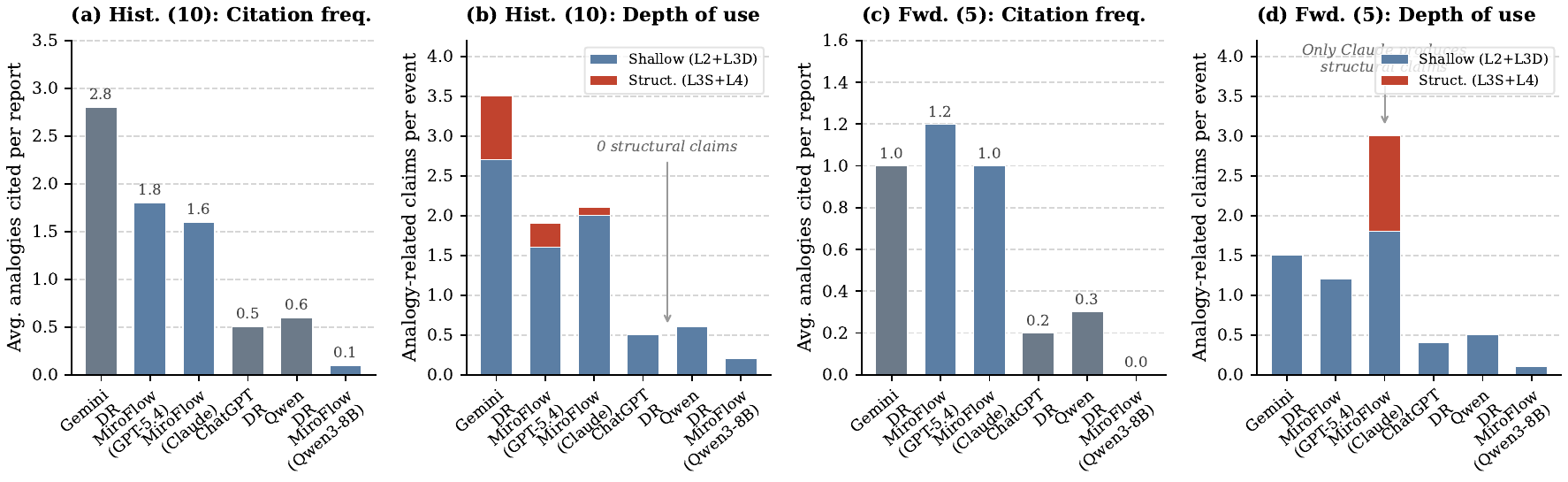}
    \caption{Preliminary results of \dr agents in \adr tasks.}
    \label{fig:adr_prelim}
\end{figure}

\textbf{Preliminary results.}
We evaluate four commercial \dr agents on all $15$ events using the natural prompt: \chatgpt~\citep{openai2025deepresearch}, \gemini~\citep{google2024geminideepresearch}, \qwen~\citep{tongyi2025deepresearch}, and \miro~\citep{miromind2026miroflow} with three backbone LLMs (\texttt{Claude Sonnet 4.5}~\citep{anthropic2025claude45}, \texttt{GPT-5.4}~\citep{openai2026gpt54}, and \texttt{Qwen3-8B}~\citep{yang2025qwen3}). The results are given in Fig.~\ref{fig:adr_prelim}, where one can find that most of existing \dr agents fail to find useful analogies in foresight analysis. The relatively most capable agent \gemini, can find limited analogies, but fail to produce in-depth analysis by effectively leveraging the analogies.

\subsection{Theoretical Discussions}
\label{sec:theory-disc}
To understand why existing \dr agents fail to identify useful analogies to perform foresight analysis, we provide the following theoretical discussions.
To begin with, most existing \dr agents are trained on planning and synthesis queries and results at the surface-level descriptions~\citep{huang2025deep}, while finding historical analogies require structural alignment of the underlying objects, roles and mechanisms.
\begin{theorem}[Surface Non-Identifiability]
\label{thm:nonid}
Suppose there exist two admissible worlds $w^0, w^1$ with identical surface observations $\mathcal{O}_D^{\infty}(w^0) = \mathcal{O}_D^{\infty}(w^1)$, but for a hidden position $s$, $Z_s(w^0) \neq Z_s(w^1)$ where $Z_s(w) = \mathbbm{1}[s \in \mathcal{S}_T^{\mathrm{true}}(w)]$ that $s$ does not align between $w^0$ and $w^1$,
Then, no surface-level method can identify whether $s$ is active: $\inf_{\hat{Z}_s} \max_{j} \Pr_{w^j}(\hat{Z}_s \neq Z_s(w^j)) \geq 1/2$. If the two worlds imply separated foresight distributions $\mathrm{TV}(P_F^{w^0}, P_F^{w^1}) \geq 2\Delta_s$, then any surface-level forecast satisfies $\max_j \mathrm{TV}(\widehat{P}_F, P_F^{w^j}) \geq \Delta_s$. This holds for any number of analogies retrieved.
\end{theorem}
The full statement and the proof of Theorem~\ref{thm:nonid} is given in Appendix~\ref{appdx:theory}. Intuitively, Theorem~\ref{thm:nonid} implies that aligning analogies for foresight analysis requires aligning the underlying structural roles instead of the surface-level descriptions.
It also explains why it is widely observed in the literature that LLMs are struggling with finding the analogies~\citep{sourati2024arn,li2025past}. 

Nevertheless, given the partial observability of both the target events and the potential source events, for $K$ retrieved analogies, we may only align a few positions, which may suffer from misalignment issues, since multiple events may coincidentally align with each other.
\begin{theorem}[Cross-Analogy Confirmation]
\label{thm:recovery}
Let $X_{k,s} = \mathbbm{1}[\text{analogy } E_k \text{ confirms position } s]$ with $P(X_{k,s}=1 \mid Z_s=1) = q_{k,s}$ and $P(X_{k,s}=1 \mid Z_s=0) = p_{k,s}$, conditionally independent across $k$ given $Z_s$, then
$
\log \frac{P(Z_s = 1 \mid X_{\mathcal{A}})}{P(Z_s = 0 \mid X_{\mathcal{A}})} = \log\frac{\pi_s}{1-\pi_s} + \sum_{k \in \mathcal{A}} \left[ X_{k,s} \log\frac{q_{k,s}}{p_{k,s}} + (1-X_{k,s}) \log\frac{1-q_{k,s}}{1-p_{k,s}} \right],
$ 
where $\pi_s = P(Z_s = 1)$ is the prior probability 
that position $s$ is active before observing any 
analogical evidence.
Each independent confirming analogy multiplies posterior odds by $q_{k,s}/p_{k,s}$, so coincidence probability decays exponentially. 
\end{theorem}
The full statement and proof of Theorem~\ref{thm:recovery} are given in Appendix~\ref{appdx:theory}. Theorem~\ref{thm:recovery} states the conditions in order to sufficiently align the structures between the target events and the analogies.
For example, given $\pi_s{=}0.5, q_s{=}1, p_s{=}0.2, \delta{=}0.05$, one requires 2 independent analogies per position to distinguish the structural necessity.

\textbf{Discussions with existing causal identifiability theory.} As \adr is essentially a causal problem, one could also draw analogical relations between our theoretical results and the causal identifiability theories. For example, Theorem~\ref{thm:recovery} parallels the positive results from causal discovery on multiple datasets~\citep{huang2020cdmini} and from multi-view causal representation learning~\citep{yao2024multiview,brehmer2022weakly} where multiple partial views of shared latent structure enable identification of the underlying causal structure. 

\textbf{Design implications.} One could draw practical implications from the two theoretical results as well. Specifically, we can derive two principles:  \textbf{Principle 1} (Theorem~\ref{thm:nonid}): analogy retrieval should operate on $M(E)$ rather than $D(E)$, since surface observations collapse mechanism-distinct worlds;
and \textbf{Principle 2} (Theorem~\ref{thm:recovery}): cross-analogy confirmation with $K_s^* \geq 2$ independent analogies per position is needed to distinguish structural necessity from coincidence. 

%% file: 3_ADR-agent.tex
\section{\adrfull Agent}
\label{sec:method}
Built upon discussions in Sec.~\ref{sec:theory}, we propose \oursfull (\ours) to realize the two principles, via (i) structural decomposed representation; and (ii) structural reflective generation. An overview of \ours is illustrated in Fig.~\ref{fig:framework}.

\begin{figure}[t]
    \centering
    \includegraphics[width=\linewidth]{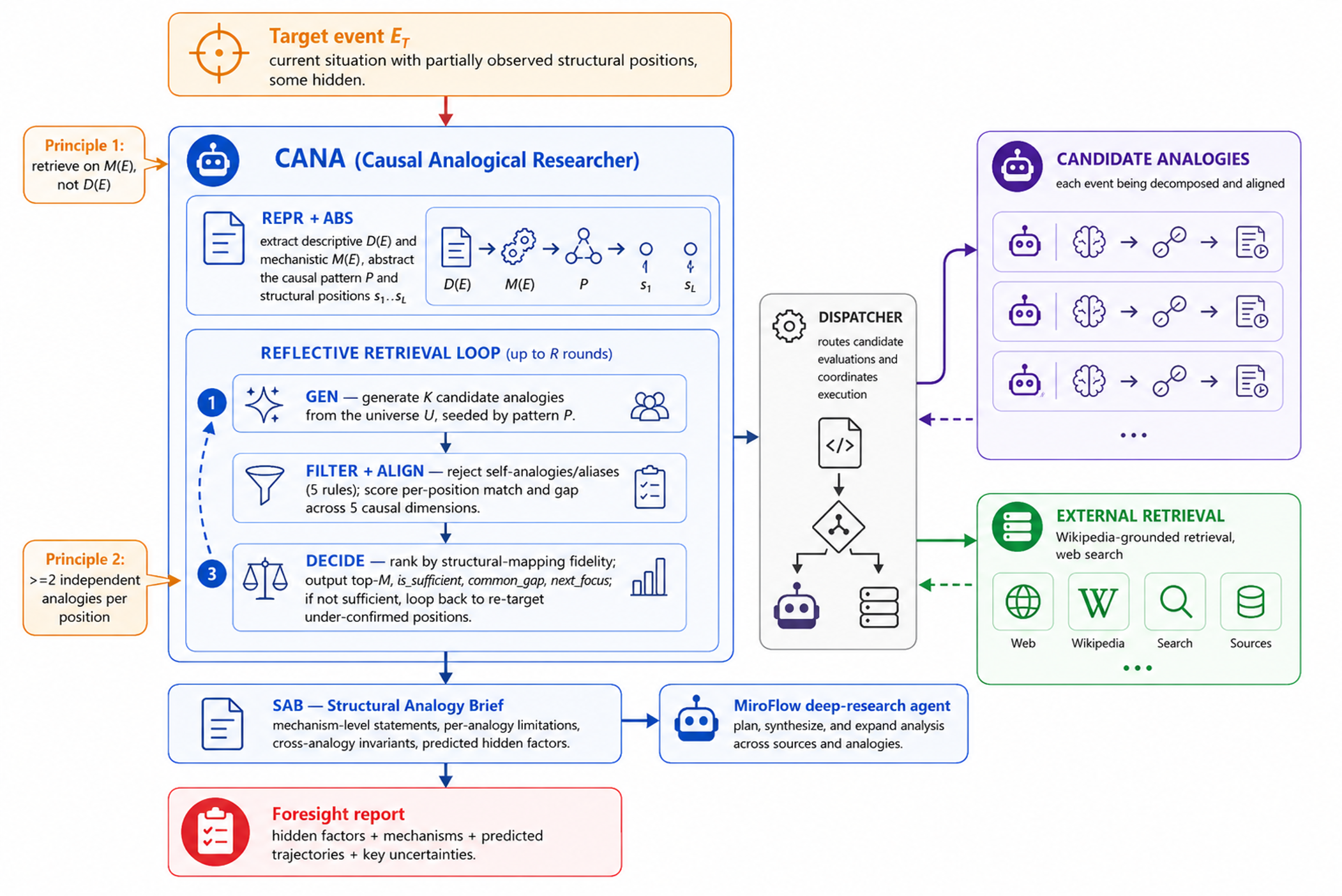}
    \caption{Overview of \oursfull (\ours). \ours decomposes each event
    into a structural representation (\texttt{preconditions, temporal
    chains, mechanisms, outcomes}), realizing \textbf{Principle~1}
    (retrieve on the mechanism $M(E)$ rather than the surface description
    $D(E)$). It then performs structural reflective generation, using
    feedback on covered and uncovered structural positions to iteratively
    retrieve cross-confirming analogies until each hidden position is
    supported by $\geq 2$ independent analogies, realizing
    \textbf{Principle~2}. The resulting Structural Analogy Brief (\sab),
    which contains mechanism statements, per-analogy limitations,
    cross-analogy invariants, and predicted hidden factors, is passed as
    context to the main \dr agent (\miro).}
    \label{fig:framework}
\end{figure}

\subsection{\oursfull}
\textbf{Structural decomposed representation.} Motivated by~\citet{clement1991systematicity}, we aim to seek a general decomposed representation of historical events, such that the follow-up structural alignment is easier. Therefore, we propose to leverage LLMs to recognize the temporal roles of events, including \texttt{preconditions, temporal chains, mechanisms, and outcomes}.
As mentioned in Def.~\ref{sec:definition} that, the positions is not unique, one may assign different roles to the objects in an event, depending on the objective and the context (e.g., the trigger--enabler--amplifier--mediator--outcome form in Fig.~\ref{fig:ADR_graph}). Since events usually unfold in a temporal order~\citep{HistoricalAnalogies}, we find that our decomposition is empirically general and useful. Detailed prompts are given in Appendix~\ref{appdx:adr_gen}.

\textbf{Structural reflective generation.} 
Following~\citet{li2025past} that uses parametric knowledge to retrieve historical analogies, we feed structurally decomposed target events to LLMs and ask the LLM to come up with relative analogies. To realize the second principle, we also construct feedback signals to prompt the LLM to iteratively reflect and self-improve upon the generation results. Following the insights given by Theorem~\ref{thm:recovery}, we prompt LLMs to self-critique and reflect based on the covered and not covered aspects of the target events. After a few rounds, one could observe more cross-confirmed events and find the desired historical analogies.

Although \citet{li2025past} also proposed a decomposition of events and leveraged reflective signals to improve the analogy retrieval, the lack of theoretical foundations and principles make the generated results remain focused on surface features, as we will show in the experiments.
 

\subsection{Agentic Modules for Deep Research}
With the \ours developed for effective historical analogy generation, we further incorporate \ours as a lightweight module in the state-of-the-art open-source \dr framework \miro~\citep{miromind2026miroflow}.

\ours synthesizes the Structural Analogy Brief (\sab) as additional context to the main agent \dr multi-agent framework. The construction of \sab relies on the \ours decomposition. A critical difference from previous structural reflective generation is that, \sab is designed find a more diverse set of analogies in order to cover a broad domain of possible analogies, and then do composition. Specifically, the \sab provides mechanism-level statements, limitations of each retrieved analogy, cross-analogy insights as well as predicted hidden factors. The rich context provides grounds for \dr.

%% file: 4_exp.tex
\input{tables/MDS.tex}
\input{tables/v4.tex}
\section{Experiments}
\label{sec:experiments}
We conduct experiments using both benchmarks from~\citet{li2025past} as well as \adrb, in order to verify the effectiveness of \ours in analogy generation, and in \adr task, respectively.

\subsection{Historical Analogy Generation}
\label{sec:acl}

\textbf{Setup.} We used more updated LLMs ranging in different capabilities, including \texttt{Qwen3-8B}~\citep{yang2025qwen3},  \texttt{GPT-5.4-mini}~\citep{openai2026gpt54mini}, and \texttt{Claude-Sonnet-4.6}~\citep{anthropic2026sonnet46card}. We compare the state-of-the-art methods proposed by~\citet{li2025past} (i.e., self-reflection), direct generation, and summary-then-generation. 

\textbf{Evaluation.} We consider two sets of evaluation protocols, one is the topic-background-process-result (MDS) decomposition by~\citet{li2025past}, and the other is motivated by \ours decomposition, where we consider more dimensions involving the historical analogy alignment, including chain isomorphism, direction consistency, mapping consistency, overlap sufficiency, idiosyncrasy coverage, system alignment, novelty, surface similarity, and calibration bonus. Detailed rubrics for evaluating the 9-dimension are given in Appendix~\ref{appdx:his_gen_eval}.
 
\textbf{Results.} Results using \citet{li2025past} and our 9-item finegrained evaluation are given in Table~\ref{tab:mds-result}, Table~\ref{tab:v4-result}, respectively. 
We can find that, although self-reflection by \citet{li2025past} uses a more aligned decomposition as MDS evaluation protocol, \ours still achieves the best performance across all LLMs. Interestingly, \texttt{GPT-5.4-mini} can achieve better results than self-reflection when using summary-then-generation.

Moreover, one could also notice that when using the relatively powerful LLM like \texttt{Claude-Sonnet-4.6}, MDS scores saturate to around $4.1$ to $4.2$. That is because the decomposition of MDS does not sufficiently align with the underlying structural form of the events. In contrast, the fine-grained 9-dimension decomposition in Table~\ref{tab:v4-result} provides a better lens to understanding the differences of methods. Although \ours does not use a similar fine-grained decomposition, surprisingly, \ours achieves up to 25\% improvements than self-reflection when using \texttt{GPT-5.4-mini}. In addition, the weakest \texttt{Qwen3-8B} using \ours can surpass \texttt{GPT-5.4-mini} with self-reflection, or \texttt{Claude-Sonnet-4.6} with summarization, implying the effectiveness and generality of \ours.

\textbf{Case studies.} Concretely, in Table~\ref{tab:failure} in Appendix~\ref{appdx:failure}, we present examples showing the seven systematic failure modes of descriptive matching via MDS decomposition, each resolved by \ours decomposition. It also echoes the importance of alignment via $M(E)$ as predicted by our theory.

\subsection{\adrfull Evaluation}
 
\textbf{Setup.} 
We evaluate $10$ agent configurations on all $15$ \adrb events:
$3$ commercial \dr agents (\chatgpt~\citep{openai2025deepresearch}, \gemini~\citep{google2024geminideepresearch}, \qwen~\citep{tongyi2025deepresearch});
$3$ vanilla \miro agents~\citep{miromind2026miroflow} with \texttt{Claude Sonnet 4.5}~\citep{anthropic2025claude45}, \texttt{GPT-5.4}~\citep{openai2026gpt54}, and \texttt{Qwen3-8B}~\citep{yang2025qwen3}; and
$3$ \ours agents with the same three backbones.
All agents receive the identical prompts as given by \adrb. Scoring uses the rubric-based evaluation with \texttt{Claude Sonnet 4.5}. We report the evaluation metrics as introduced in Sec.~\ref{sec:benchmark}.

\input{tables/ard_results.tex}
Table~\ref{tab:main_results} shows the main results across all $15$ events. From the results, we can find that

\textbf{Finding 1: \dr agents are struggling in foresight analysis with historical analogies.}
Complementary to previous preliminary results, we can find that most \dr agents fail to write in-depth foresight analysis, and the reports given by \dr agents lack of grounding mechanisms.

\textbf{Finding 2: \ours improves significantly for \adr.}
No matter which backbone LLM is used, one could find that \miro equipped with \ours show great improvements than both previous commercial \dr agents, as well as \miro agents without \ours. 
Even using the weakest LLM \texttt{Qwen3-8B}, \miro with \ours achieves commercial \dr agent level performance and even surpasses.
 
\textbf{Finding 3: Complementary strengths across backbones.}
Different LLMs with \ours show different characteristics.
Interestingly, \ours with \texttt{Sonnet-4.5} comes up with substantial  volume of claims ($223$ L3-S+L4) and hidden-factors ($16/42$ HF@L4). \ours with \texttt{GPT-5.4} focuses on high-level mechanism grounding synthesis, as well as most in-depth foresight quality. Both achieve RAS $= 1.00$. This suggests that the lightweight \sab with \ours provides the structural reasoning scaffold, while the backbone model determines execution quality.
 
\textbf{Connecting to theory.}
The HF@L4 $= 0/42$ result for all commercial agents validates Theorem~\ref{thm:nonid}: surface-level methods cannot identify hidden structural positions regardless of model capability. The lift from \sab validates Theorem~\ref{thm:recovery}: cross-analogy mechanism confirmation enables hidden factor identification and high-quality foresight analysis. 

\input{Figures/case.tex}

\textbf{Case studies on forward event.} We present a case studies of a forward event that ask agents to analyze the macro financial trends given by the Iran War 2026 and the new high of S\&P 500 when the oil prices remain significantly high.

Every agent addresses this anomaly, but with dramatically different analytical depth. As shown in Fig.~\ref{fig:case_study}, existing \dr agents cite descriptive evidence (L1), and explains that it is because of the ceasefire effects. Nevertheless, \ours cites the historical analogies that ``across 1973, 1979, 1990--91, equity inflection \textsc{precedes} supply normalization by 6--18 months'', which provides more in-depth reasoning underlying the surface, aligned with the application of historical analogies by humans~\citep{HistoricalAnalogies}. 

%% file: tables/MDS.tex
\newcolumntype{a}{>{\columncolor{black!8}\centering\arraybackslash}p{0.85cm}} 
\newcolumntype{b}{>{\columncolor{red!10!gray!10}\centering\arraybackslash}p{0.85cm}}
\newcolumntype{d}{>{\columncolor{teal!12!gray!10}\centering\arraybackslash}p{0.85cm}}
\newcolumntype{q}{>{\columncolor{violet!12!gray!10}\centering\arraybackslash}p{0.85cm}}
\newcolumntype{s}{>{\columncolor{brown!15!white!50}\centering\arraybackslash}p{0.85cm}}
\newcolumntype{e}{>{\columncolor{blue!25!gray!15}\centering\arraybackslash}p{0.85cm}}
\renewcommand{\arraystretch}{1.15} 
\setlength\tabcolsep{1.3pt}
\begin{table*}[!t]
  \centering
  \caption{MDS evaluation results on General Analogies with the Judge as GPT-5.4. ``T'', ``B'', ``P'', ``R'' denote the four descriptive dimensions of Topic, Background, Process, Result. ``Abs'' is abstract similarity; ``Lit'' is literal similarity via Jaccard token overlap, lower is better; ``MDS'' is the weighted composite TBPR scores with penalty for events flagged as near\_self by the alias detector. Best in each LLM block is \textbf{bolded}, second best \uline{underlined}.}
  \small
  \begin{tabular}{laaabbbdddqqqe}
  \toprule
  \multicolumn{1}{c}{\textbf{Method}} & \textbf{T$_{\texttt{Abs}}$} & \textbf{T$_{\texttt{Lit}}$} & \textbf{T$_{\texttt{All}}$}  & \textbf{B$_{\texttt{Abs}}$} & \textbf{B$_{\texttt{Lit}}$} & \textbf{B$_{\texttt{All}}$}  & \textbf{P$_{\texttt{Abs}}$} & \textbf{P$_{\texttt{Lit}}$} & \textbf{P$_{\texttt{All}}$}  & \textbf{R$_{\texttt{Abs}}$} & \textbf{R$_{\texttt{Lit}}$} & \textbf{R$_{\texttt{All}}$}  & \textbf{MDS} \\
  \midrule
\multicolumn{13}{c}{gpt-5.4-mini} \\
\cmidrule{1-13}
  Direct Gen.      & 2.24 & \textbf{0.14} & 0.46 & 2.36 & \uline{0.11} & 0.53 & 2.38 & \textbf{0.09} & 0.60 & 2.32 & \textbf{0.10} & 0.55 & 3.05 \\
  Summarizing      & \uline{2.99} & 0.18 & \uline{0.51} & \textbf{3.03} & 0.12 & \uline{0.67} & \uline{3.03} & 0.11 & \uline{0.69} & \textbf{2.95} & 0.13 & \uline{0.64} & \uline{3.59} \\
  Self-reflection  & 2.70 & 0.23 & 0.45 & \uline{3.00} & 0.18 & 0.59 & 3.00 & 0.15 & 0.66 & 2.78 & 0.17 & 0.60 & 3.34 \\
  \ours             & \textbf{3.13} & \uline{0.14} & \textbf{0.63} & 2.98 & \textbf{0.11} & \textbf{0.71} & \textbf{3.10} & \uline{0.09} & \textbf{0.79} & \textbf{2.95} & \uline{0.10} & \textbf{0.70} & \textbf{4.01} \\
\cmidrule{1-13} \multicolumn{13}{c}{sonnet-4-6} \\
\cmidrule{1-13}
  Direct Gen.      & 1.46 & 0.19 & 0.29 & 2.23 & 0.15 & 0.42 & 2.12 & 0.14 & 0.43 & 2.21 & 0.14 & 0.41 & 2.24 \\
  Summarizing      & \uline{3.16} & 0.16 & 0.59 & \uline{3.00} & \uline{0.10} & 0.72 & \uline{3.08} & 0.10 & 0.76 & \uline{3.06} & 0.11 & 0.70 & 3.91 \\
  Self-reflection  & \textbf{3.36} & \uline{0.15} & \uline{0.65} & \textbf{3.04} & 0.10 & \uline{0.73} & \textbf{3.09} & \uline{0.09} & \uline{0.78} & \textbf{3.10} & \uline{0.10} & \textbf{0.75} & \uline{4.11} \\
  \ours             & \uline{3.16} & \textbf{0.13} & \textbf{0.69} & 2.93 & \textbf{0.09} & \textbf{0.75} & 3.02 & \textbf{0.08} & \textbf{0.79} & 2.96 & \textbf{0.10} & \uline{0.74} & \textbf{4.17} \\
\cmidrule{1-13} \multicolumn{13}{c}{qwen3-8b} \\
\cmidrule{1-13}
  Direct Gen.      & 1.00 & \textbf{0.05} & 0.30 & 1.09 & \textbf{0.05} & 0.32 & 1.11 & \textbf{0.05} & 0.33 & 1.09 & \textbf{0.05} & 0.32 & 1.76 \\
  Summarizing      & 2.71 & 0.17 & 0.53 & 2.76 & 0.12 & 0.63 & 2.68 & 0.11 & 0.66 & 2.69 & 0.12 & 0.62 & 3.47 \\
  Self-reflection  & \textbf{3.06} & 0.15 & \uline{0.61} & \textbf{2.96} & 0.10 & \uline{0.70} & \textbf{3.03} & 0.09 & \textbf{0.77} & \textbf{3.02} & 0.10 & \uline{0.70} & \uline{3.94} \\
  \ours             & \uline{3.03} & \uline{0.12} & \textbf{0.68} & \uline{2.82} & \uline{0.08} & \textbf{0.74} & \uline{2.80} & \uline{0.08} & \uline{0.76} & \uline{2.91} & \uline{0.09} & \textbf{0.75} & \textbf{4.11} \\
  \bottomrule
  \end{tabular}%
  \label{tab:mds-result}%
\vspace{-0.15in}
\end{table*}%

%% file: tables/v4.tex

\newcolumntype{a}{>{\columncolor{black!8}\centering\arraybackslash}p{0.85cm}} 
\newcolumntype{b}{>{\columncolor{red!10!gray!10}\centering\arraybackslash}p{0.85cm}}
\newcolumntype{d}{>{\columncolor{teal!12!gray!10}\centering\arraybackslash}p{0.85cm}}
\newcolumntype{q}{>{\columncolor{violet!12!gray!10}\centering\arraybackslash}p{0.85cm}}
\newcolumntype{s}{>{\columncolor{brown!15!white!50}\centering\arraybackslash}p{0.85cm}}
\newcolumntype{e}{>{\columncolor{blue!25!gray!15}\centering\arraybackslash}p{0.85cm}}
\renewcommand{\arraystretch}{1.0}
\setlength\tabcolsep{3pt} 
\begin{table*}[!t]
  \centering
  \caption{Detailed evaluation results on General Analogies using a 9-item structured rubric. 
The column headers denote: 
\textbf{C. Isom.} (Chain Isomorphism), 
\textbf{Dir. Con.} (Direction Consistency), 
\textbf{Map. Con.} (Mapping Consistency), 
\textbf{Over. Suf.} (Overlap Sufficiency), 
\textbf{Idio. Cov.} (Idiosyncrasy Coverage), 
\textbf{Sys. Alig.} (System Alignment), 
\textbf{Nov.} (Novelty), 
\textbf{Surf. Sim.} (Surface Similarity), and 
\textbf{Calib. B.} (Calibration Bonus). 
Final represents the weighted composite score. 
All evaluations are conducted by GPT-5.4. 
Best results within each LLM block are \textbf{bolded}, and second best are \uline{underlined}.}
  \small
  \begin{tabular}{l aaa bb d q ss e}
  \toprule
  \textbf{Method} & 
  \makecell[b]{\textbf{C.}\\\textbf{Isom.}} & \makecell[b]{\textbf{Dir.}\\\textbf{Con.}} & \makecell[b]{\textbf{Map.}\\\textbf{Con.}} & 
  \makecell[b]{\textbf{Over.}\\\textbf{Suf.}} & \makecell[b]{\textbf{Idio.}\\\textbf{Cov.}} & 
  \makecell[b]{\textbf{Sys.}\\\textbf{Alig.}} & \makecell[b]{\textbf{Nov.}\\ } & 
  \makecell[b]{\textbf{Surf.}\\\textbf{Sim.}} & \makecell[b]{\textbf{Calib.}\\\textbf{B.}} & \textbf{Final} \\
  \midrule
\multicolumn{11}{c}{\textit{gpt-5.4-mini}} \\ 
\midrule
  Direct Gen.      & 1.46 & 1.77 & 1.68 & 1.84 & 0.96 & 0.95 & \textbf{2.07} & 2.02 & -0.69 & 3.31 \\
  Summarizing      & \textbf{1.96} & \textbf{2.16} & \textbf{2.26} & \uline{2.75} & 1.52 & \uline{1.79} & 1.47 & 2.77 & -0.65 & \uline{4.49} \\
  Self-reflection  & 1.94 & 2.10 & \uline{2.23} & 2.64 & \textbf{1.57} & \textbf{1.88} & 1.18 & 2.77 & -0.58 & 4.12 \\
  \ours            & \uline{1.95} & \uline{2.14} & 2.20 & \textbf{2.77} & \uline{1.54} & 1.78 & \uline{1.84} & 2.64 & -0.57 & \textbf{5.22} \\
\midrule
\multicolumn{11}{c}{\textit{sonnet-4-6}} \\
\midrule
  Direct Gen.      & 1.69 & 2.07 & 1.96 & 2.07 & 1.29 & 1.48 & \textbf{2.99} & 2.19 & -0.46 & 3.43 \\
  Summarizing      & 1.89 & 2.05 & \uline{2.21} & 2.66 & 1.41 & 1.68 & 1.58 & 2.71 & -0.67 & 4.65 \\
  Self-reflection  & \uline{1.97} & \textbf{2.15} & \textbf{2.24} & \uline{2.74} & \uline{1.50} & \uline{1.88} & 1.73 & 2.73 & -0.56 & \uline{5.23} \\
  \ours            & \textbf{1.98} & \uline{2.11} & 2.21 & \textbf{2.84} & \textbf{1.52} & \textbf{1.91} & \uline{1.89} & 2.62 & -0.47 & \textbf{5.70} \\
\midrule
\multicolumn{11}{c}{\textit{qwen3-8b}} \\
\midrule
  Direct Gen.      & 0.855 & 1.18 & 0.59 & 0.26 & 0.11 & 0.10 & \textbf{2.94} & 0.36 & -0.17 & 2.14 \\
  Summarizing      & \textbf{1.88} & \textbf{2.08} & \uline{2.11} & 2.53 & \uline{1.39} & \uline{1.64} & 1.79 & 2.64 & -0.66 & 4.37 \\
  Self-reflection  & \uline{1.87} & \uline{2.04} & \textbf{2.11} & \textbf{2.66} & \textbf{1.43} & \textbf{1.66} & 1.73 & 2.69 & -0.68 & \uline{4.63} \\
  \ours            & 1.77 & 1.94 & 1.98 & \uline{2.66} & 1.34 & 1.49 & \uline{2.06} & 2.54 & -0.70 & \textbf{4.87} \\
  \bottomrule
  \end{tabular}
  \label{tab:v4-result}
  \vspace{-0.05in}
\end{table*}

%% file: tables/ard_results.tex
\newcolumntype{s}{>{\columncolor{cyan!8}\centering\arraybackslash}c}
\newcolumntype{q}{>{\columncolor{gray!10}\centering\arraybackslash}c}
\newcolumntype{Z}{>{\centering\arraybackslash}p{0.25cm}}


\begin{table*}[!htbp]
\vspace{-0.05in}
\centering\small
\caption{Results on \adrb. Left: $10$ historical events ($22$ hidden factors). Right: $5$ forward events ($20$ hidden factors, outcomes unfolding). L3-S+L4 = cross-analogy structural claims per event; HF@L4 = hidden factors inferred via analogical reasoning; RAS = analysis quality; FQS$_d$ = depth-filtered foresight quality; Gr\% = fraction of claims with mechanism grounding $\geq 2$. All commercial \dr agents produce \textbf{zero} HF@L4 across both splits.}
\label{tab:main_results}
\renewcommand{\arraystretch}{1.3} 
\setlength{\tabcolsep}{4.5pt}     
\resizebox{\textwidth}{!}{
\begin{tabular}{@{} ll ssqqq Z ssqqq @{}}
\toprule
& & \multicolumn{5}{c}{\textbf{Historical events (10)}} & & \multicolumn{5}{c}{\textbf{Forward events (5)}} \\
\cmidrule(lr){3-7} \cmidrule(lr){9-13}
& \textbf{Agent} & 
\textbf{L3S+L4} & \textbf{HF@L4} & \textbf{RAS} & \textbf{FQS$_d$} & \textbf{Gr\%} & & 
\textbf{L3S+L4} & \textbf{HF@L4} & \textbf{RAS} & \textbf{FQS$_d$} & \textbf{Gr\%} \\
\midrule
\multirow{3}{*}{\rotatebox{90}{\textbf{\textsf{Comm.}}}}
& Gemini DR        &  0.8 &  0/22 & 0.967 & 0.71 & 0.18 & &  0.0 & 0/20 & 0.889 & 0.79 & 0.21 \\
& ChatGPT DR       &  0.0 &  0/22 & 0.741 & 0.53 & 0.13 & &  0.0 & 0/20 & 0.630 & 0.47 & 0.07 \\
& Qwen DR          &  0.0 &  0/22 & 0.689 & 0.37 & 0.15 & &  0.0 & 0/20 & 0.600 & 0.27 & 0.08 \\
\midrule
\multirow{3}{*}{\rotatebox{90}{\textbf{\textsf{Vanilla}}}}
& \miro (GPT-5.4)  &  0.3 &  1/22 & 0.933 & 0.73 & 0.34 & &  0.0 & 0/20 & 0.837 & 0.86 & 0.42 \\
& \miro (Sonnet-4.5)    &  0.1 &  0/22 & 0.819 & 0.82 & 0.35 & &  1.2 & 1/20 & 0.800 & 0.55 & 0.12 \\
& \miro (Qwen3-8B) &  0.0 &  0/22 & 0.319 & 0.29 & 0.07 & &  0.0 & 0/20 & 0.570 & 0.00 & 0.00 \\
\midrule
\multirow{3}{*}{\rotatebox{90}{\textbf{\textsf{Ours}}}}
& \ours (Sonnet-4.5)   & \textbf{14.0} & \textbf{13/22} & \textbf{1.000} & 0.77 & 0.74 & & \textbf{16.6} & 3/20 & \textbf{1.000} & 0.56 & 0.84 \\
& \ours (GPT-5.4)  & 11.6 &  7/22 & 0.993 & \textbf{0.92} & \textbf{0.94} & & 13.4 & 3/20 & \textbf{1.000} & \textbf{0.86} & \textbf{0.94} \\
& \ours (Qwen3-8B) &  3.2 &  1/22 & 0.607 & 0.82 & 0.81 & &  6.0 & \textbf{3/20} & 0.807 & 0.62 & 0.81 \\
\bottomrule
\end{tabular}}
\end{table*}

%% file: Figures/case.tex
\begin{figure*}[t]
\centering
\vspace{-0.15in}

\definecolor{oursColor}{RGB}{193, 67, 46}       
\definecolor{oursLight}{RGB}{253, 242, 240}     
\definecolor{baseColor}{RGB}{115, 125, 135}     
\definecolor{baseLight}{RGB}{248, 249, 250}     
\definecolor{hlEvent}{RGB}{31, 97, 141}         

\newcommand{\leveltag}[2]{%
  \tikz[baseline=(tag.base)]{\node[fill=#1, text=white, font=\sffamily\scriptsize\bfseries, 
    rounded corners=3pt, inner xsep=4pt, inner ysep=2.5pt] (tag) {#2};}%
}
\resizebox{\textwidth}{!}{
\begin{tikzpicture}[
  mybox/.style 2 args={
    rectangle, rounded corners=6pt, draw=#1, fill=#2, 
    line width=0.8pt, text width=7.4cm, align=left, 
    inner xsep=10pt, inner ysep=10pt, font=\small,
    drop shadow={opacity=0.08, shadow xshift=2pt, shadow yshift=-2pt}
  },
  widebox/.style 2 args={
    rectangle, rounded corners=6pt, draw=#1, fill=#2, 
    line width=0.8pt, text width=15.5cm, align=left, 
    inner xsep=10pt, inner ysep=10pt, font=\small,
    drop shadow={opacity=0.08, shadow xshift=2pt, shadow yshift=-2pt}
  }
]

\node[font=\sffamily\large\bfseries, text=black!85, anchor=west] at (-8.1, 0.6) 
  {Case Study: ``Why is the S\&P near highs despite the energy shock?''};

\node[widebox={black!20}{black!2}, anchor=north] (Q) at (-0.2, 0)
  {\textit{\textbf{Scenario (Iran War 2026, $t_c$ = May 2026):} Strait of Hormuz closed, oil $-$10.1 mb/d, Brent \$118/bbl.}}; 

\node[font=\sffamily\normalsize\bfseries, text=baseColor, anchor=south west] at (-8.1, -1.9) {\dr agents};
\node[font=\sffamily\normalsize\bfseries, text=oursColor, anchor=south west] at (0.3, -1.9) {\ours agents};


\node[mybox={baseColor}{baseLight}, anchor=north west] (G) at (-8.1, -2.1) {
  \leveltag{baseColor}{L1} \hspace{2pt} \textbf{Gemini DR} \hfill \textcolor{baseColor}{\scriptsize\textit{39KB report, 0 events cited}}\\[5pt]
  ``A striking anomaly... equities staged a rapid V-shaped recovery. 
  [...] There is a \underline{natural lag effect} inherent in energy shocks. 
  Elevated costs have not yet permeated income statements.''\\[4pt]
  {\scriptsize\textcolor{baseColor}{\textit{$\rightarrow$ Generic mechanism, no historical anchor or temporal pattern.}}}
};

\node[mybox={baseColor}{baseLight}, anchor=north west] (CQ) at (-8.1, -6.0) {
  \leveltag{baseColor}{L1} \hspace{2pt} \textbf{ChatGPT DR / Qwen DR} \hfill \textcolor{baseColor}{\scriptsize\textit{0 events cited}}\\[5pt]
  Risks enumerated generically: ``ceasefire collapse, OPEC fragmentation.''\\
  {\scriptsize\textcolor{baseColor}{\textit{$\rightarrow$ No derivation from historical precedent.}}}
};


\node[mybox={oursColor}{oursLight}, anchor=north west] (OC) at (0.3, -2.1) {
  \leveltag{oursColor}{L3-S} \hspace{2pt} \textbf{\ours\ (Claude)} \hfill \textcolor{oursColor}{\scriptsize\textit{4 mechanisms $\times$ 3+ events}}\\[5pt]
  ``Equity market inflection \textsc{precedes} physical supply normalization by 6--18 months 
  when forward curves signal credible resolution 
  (\textcolor{hlEvent}{\textbf{1973 embargo}}, \textcolor{hlEvent}{\textbf{1990--91 Gulf War}}, 
  \textcolor{hlEvent}{\textbf{1980--88 Tanker War}}). 
  The April 8 ceasefire [...] triggered appreciation because investors price 
  normalization 6--12 months forward.''
};

\node[mybox={oursColor}{oursLight}, anchor=north west] (OQ) at (0.3, -6.0) {
  \leveltag{oursColor}{L3-S} \hspace{2pt} \textbf{\ours\ (Qwen3-8B)} \hfill \textcolor{oursColor}{\scriptsize\textit{8B model, 5KB report}}\\[5pt]
  ``S\&P hitting new highs reflects 
  \textcolor{hlEvent}{\textbf{1973 Oil Crisis}} and \textcolor{hlEvent}{\textbf{2020--21 tech surge}}, 
  where euphoria decoupled from fundamentals. 
  Aligns with invariant: \textit{`equity inflection precedes supply normalization by 6--18 months.'}''
};


\end{tikzpicture}
}
\caption{How different agents handle the same analytical question (Iran War 2026). 
\textbf{Left} (gray): baseline agents produce L1 claims regardless of report length or model capability. 
\textbf{Right} (red): \ours agents produce L3-S claims grounded in $\geq$2 historical events.}
\label{fig:case_study}
\vspace{-0.15in}
\end{figure*}

%% file: 9_appdx.tex
\newpage
\appendix
\onecolumn

\section{Related Work}
\label{appdx:related}

\textbf{Historical analogies for foresight analysis.}
In human history, historical analogies are widely used or cited for foresight~\citep{achenbaum1983making,guldi2014history,parsons2016historical,ghilani2017looking,HistoricalAnalogies}. For example, policymakers often exploit analogies to reason about unprecedented scenarios~\citep{neustadt1986thinking,houghton1996role,khong1992analogies,brunk2008curing}. Scientists use analogical reasoning to identify useful hypotheses~\citep{dunbar1995invivo,nersessian2008creating}. \citet{green2007structured} also showed that forecasting with analogies can be largely improved from $32\%$ to $60\%$.
Analogies refer to events that share similar mechanisms or aligned structures~\citep{gentner1983structure,falkenhainer1989sme}. Structural alignment of analogies helps with foresight analysis~\citep{bartha2010parallel}, which is inherently a causal problem where separate partial observations are aligned to recover the underlying causal relations~\citep{huang2020cdmini,adams2021identification,yao2024multiview}. We exploit the conceptual connection between analogy identification and causal learning to build our framework.

\textbf{Analogical reasoning with LLMs.}
Due to the usefulness of analogies, there is a growing body of work using LLMs for analogical reasoning~\citep{webb2023emergent}. \citet{opielka2025analogical} showed LLMs can abstract concepts and perform analogical reasoning. \citet{yasunaga2024analogical} found that self-generated analogous examples help improve LLM reasoning. \citet{yu2024thought} showed that relating analogical sub-problems helps solve a challenging composite problem with LLMs.
As LLMs are trained on extensive datasets, finding analogies with LLMs is also of great interest~\citep{jiayang2023storyanalogy,yuan2024analogykb,sourati2024arn,li2025past}.
\citet{jiayang2023storyanalogy} found that LLMs are limited in finding analogies from stories. \citet{sourati2024arn} showed limitations in finding cross-domain analogies. \citet{ye2024analobench} provided evidence on the limitations of LLMs in finding analogies in long contexts.
The most relevant work is \citet{li2025past}, which introduced a historical-analogy benchmark and showed that reflecting on topics, background, process, and result can improve historical analogy generation. None of these works fully define the problem and the principles required for historical analogy identification and integration.

\textbf{Deep research and forecasting with LLMs.}
Deep research (\dr) agents systematically equip LLMs with planning and tool use~\citep{Xu2025ACS} for complex real-world problems requiring multi-step reasoning~\citep{openai2025deepresearch,li2025webthinker}. For example, \citet{li2025webthinker,miromind2026miroflow} showed that synthesis of comprehensive search results helps tackle challenging reasoning tasks~\citep{wei2025browsecomp,mialon2023gaia,phan2025humanity} and even future prediction~\citep{zeng2026futurex}. Forecasting is inherently challenging~\citep{karger2025forecastbench} but valuable downstream~\citep{yu2023finmem,zhang2024finagent,sen2026llms}. Despite the importance of historical analogies, we find that \dr agents do not proactively exploit them, and propose the first \adr agent.

\section*{LLM Use Statement}
From the research side, this work studies the use of LLMs for analogical deep research. LLMs are also used to help do evaluation and experimentation. From the paper writing side, we use LLMs to assist with improving the writing of this work.

\section*{Limitation and Future Works} 
As a pilot work in \adr, the scale of \adrb might be a bit limited, and the evaluation also relies on the LLMs. Future works could scale up the curation of \adrb as well as more fine-grained and human-aligned evaluation protocols.

\section*{Broader impacts}
We study using LLMs for \adrfull that will benefit the whole humanity and society. This work does not involve human subjects or personally identifiable information beyond public benchmarks used under their licenses. 


\section{Additional Technical Details}
\label{appdx:technical}

\subsection{Notation}
\label{appdx:notation}
\begin{table}[ht]
\centering
\small
\caption{Notation for the ADR theory.}
\label{tab:adr-notation}
\begin{tabular}{ll}
\hline
Symbol & Meaning \\
\hline
$E_T$ & Target event \\
$\mathcal U$ & Candidate analogy universe \\
$D(E)$ & Descriptive representation of event $E$ \\
$M(E)$ & Mechanistic representation of event $E$ \\
$\mathfrak P$ & Abstract causal pattern \\
$\mathcal S_{\mathfrak P}$ & Structural positions in pattern $\mathfrak P$ \\
$\psi_E$ & Instantiation map from event factors to structural positions \\
$\phi_k$ & Structural mapping from analogy $E_k$ to target $E_T$ \\
$\Omega_{\mathrm{pos}}$ & Target-side position-level overlap \\
$\mathcal O_D^\infty$ & Full surface observation \\
$\mathcal O_M^\infty$ & Full mechanism-level observation \\
$Z_s$ & Indicator that position $s$ is truly active / structural \\
$X_{k,s}$ & Indicator that analogy $E_k$ confirms position $s$ \\
$\widehat P_s^{\mathcal A}$ & Analogy-based forecast for position $s$ \\
$P_s^w$ & True trajectory distribution at position $s$ in world $w$ \\
$P_F^w$ & True foresight distribution in world $w$ \\
$\alpha_s^{\mathrm{tr}}$ & Mechanism-transfer error bound \\
$\Delta_s$ & Position-wise two-world separation gap \\
$\Delta_F$ & Foresight-distribution two-world separation gap \\
$\delta_s$ & Posterior probability that inferred position $s$ is false \\
\hline
\end{tabular}
\end{table}
\section{More Details of the Theory}
\label{appdx:theory}

\subsection{Problem Definition}

\begin{definition}[Event Representation]
\label{def-appdx:event}
An event $E$ has two representations: (i)~\textbf{descriptive} $D(E) \in \mathbb{R}^p$, capturing observable surface features (entities, domain, timeline); and (ii)~\textbf{mechanistic} $M(E) = (V_E, G_E, \{K_{v,E}\}, \theta_E)$, where $V_E = \{v_1, \ldots, v_k\}$ are causal factors, $G_E = (V_E, \mathcal{E}_E)$ is a directed graph encoding causal relationships, $K_{v,E}$ are transition kernels, and $\theta_E$ are event-specific parameters. Each factor $v \in V_E$ has an associated trajectory $\tau_v: [0,T] \to \mathbb{R}$. The descriptive and mechanistic representations are related by a many-to-one mapping $D(E) = h(M(E), \chi_E)$, where $\chi_E$ is event-specific context.
\end{definition}

\begin{definition}[Abstract Pattern and Structural Position]
\label{def-appdx:pattern}
An abstract causal pattern is $\mathfrak{P} = (\mathcal{S}_{\mathfrak{P}}, G_{\mathfrak{P}})$, where $\mathcal{S}_{\mathfrak{P}} = \{s_1, \ldots, s_L\}$ is a finite set of structural positions and $G_{\mathfrak{P}}$ is a directed graph over $\mathcal{S}_{\mathfrak{P}}$. Event $E$ \textbf{instantiates} $\mathfrak{P}$ via a graph homomorphism $\psi_E: V_E \to \mathcal{S}_E \subseteq \mathcal{S}_{\mathfrak{P}}$, preserving edge structure. Factors $v \in V_{E_1}$ and $w \in V_{E_2}$ occupy the \textbf{same structural position} if $\psi_{E_1}(v) = \psi_{E_2}(w)$. Positions are not pre-specified but discovered through cross-event alignment, following Gentner's structure-mapping theory \citep{gentner1983structure}: objects are placed in correspondence according to their roles in the common relational structure \citep{clem1991systematicity}.
\end{definition}

\begin{definition}[Structural Mapping and Position-Level Overlap]
\label{def-appdx:mapping}
A structural mapping from source $E_S$ to target $E_T$ is a partial injection $\phi: V_{E_S} \rightharpoonup V_{E_T}$ satisfying edge preservation and position consistency: $\psi_{E_S}(v) = \psi_{E_T}(\phi(v))$ for all $v \in \mathrm{dom}(\phi)$. The \textbf{position-level overlap} for analogy set $\mathcal{A}$ is $\Omega_{\mathrm{pos}}(\mathcal{A}, E_T) = |\bigcup_k \psi_{E_T}(\mathrm{im}(\phi_k))| \,/\, |\mathcal{S}_T|$, measured on the target side.
\end{definition}

\begin{definition}[Partial Observation and Hidden Factors]
\label{def-appdx:hidden}
At cutoff $t_c$, the analyst observes $V^{\mathrm{obs}} \subseteq V_{E_T}$ and
\[
\mathcal{S}_T^{\mathrm{obs}} = \psi_{E_T}(V^{\mathrm{obs}}).
\]
The true active position set is
\[
\mathcal{S}_T^{\mathrm{true}}=\psi_{E_T}(V_{E_T}).
\]
Hidden factors are
\[
\mathcal{V}_T^{\mathrm{hid}} = V_{E_T} \setminus V^{\mathrm{obs}},
\]
and hidden positions are
\[
\mathcal{S}_T^{\mathrm{hid}}
=
\mathcal{S}_T^{\mathrm{true}}
\setminus
\mathcal{S}_T^{\mathrm{obs}}.
\]
\end{definition}

\begin{definition}[ADR Problem]
\label{def-appdx:adr}
Given target $E_T$ at cutoff $t_c$ with partial observation, and a universe $\mathcal{U}$ of historical events with fully observed representations, the ADR problem is to select $\mathcal{A}^* \subseteq \mathcal{U}$ and produce foresight report $\mathcal{F} = (\hat{\mathcal{V}}^{\mathrm{hid}}, \hat{Y}_T, \mathcal{C}, \mathcal{K})$ maximizing:
\begin{equation}
\mathcal{A}^* = \arg\max_{|\mathcal{A}| \leq K_{\max}} \; \mathrm{Qual}(\mathcal{F}(\mathcal{A})) - \mathrm{Qual}(\mathcal{F}(\emptyset))
\end{equation}
where $\mathrm{Qual}(\mathcal{F}) = \lambda_{\mathrm{cov}} \cdot \mathrm{Cov}_{\mathrm{eval}} + \lambda_{\mathrm{prec}} \cdot \mathrm{Prec}_{\mathrm{eval}} + \lambda_{\mathrm{depth}} \cdot \mathrm{Depth}_{\mathrm{eval}}$.
\end{definition}

\begin{definition}[Surface Observation and Foresight Risk]
\label{def-appdx:surface}
A \textbf{possible world} $w = (E_T, \mathcal{U}, \mathfrak{P}, \{\psi_E\}, t_c, \chi)$ specifies the target, analogy universe, pattern, instantiations, cutoff, and context. 

For two-world comparisons, we evaluate positions on a common finite candidate universe
$\mathcal{S}_\star \subseteq \mathcal{S}_{\mathfrak P}$. For each $s\in\mathcal{S}_\star$, define
\[
Z_s(w)=\mathbbm{1}[s\in\mathcal{S}_T^{\mathrm{true}}(w)].
\]
If $s$ is inactive in world $w$, $P_s^w$ denotes the corresponding null or baseline trajectory distribution.

The \textbf{full surface observation} is
\[
\mathcal{O}_D^{\infty}(w)
=
\left(
D(E_T),
\{D(E_k)\}_{E_k \in \mathcal{U}},
\{\mathrm{sim}(D(E_k), D(E_T))\}_{E_k\in\mathcal U}
\right).
\]
A \textbf{surface-level method} is any function of $\mathcal{O}_D^{\infty}(w)$.

The \textbf{full mechanism-level observation} is
\[
\mathcal{O}_M^\infty(w)
=
\left(
M^{\mathrm{obs}}(E_T),
\{M(E_k),\psi_{E_k},\phi_k\}_{E_k\in\mathcal U}
\right).
\]
A \textbf{mechanism-level method} is any function of $\mathcal{O}_M^\infty(w)$.

For a structural position $s$, let $\tau_s$ denote the trajectory induced by the target factor(s)
$v\in V_{E_T}$ with $\psi_{E_T}(v)=s$, or a fixed aggregation of those trajectories when multiple factors map to $s$.
The \textbf{foresight target} is
\[
Y_F = \Gamma(\{\tau_s(t) : s \in \mathcal{S}_\star, t > t_c\})
\]
with true distribution $P_F^w$ and forecast $\widehat{P}_F^{\mathcal{A}}$.

The \textbf{position-wise foresight risk} is:
\[
\mathcal{R}(\mathcal{A}, w)
=
\sum_{s \in \mathcal{S}_\star}
\mu_s \cdot \mathrm{TV}(\widehat{P}_s^{\mathcal{A}}, P_s^w),
\]
where $\mu_s \geq 0$, $\sum_s \mu_s = 1$.
\end{definition}

\subsection{Main Results}

We state one assumption and two theorems. Theorem~\ref{thm-appdx:nonid} is an information-theoretic impossibility result for surface matching; Theorem~\ref{thm-appdx:recovery} quantifies the benefit of cross-analogy mechanism matching. For readability, we state the results for deterministic methods. Randomized methods satisfy the same bounds in expectation over their internal randomness.

\begin{assumption}[Mechanism Transfer]
\label{ass-appdx:transfer}
If factors $v_S \in V_{E_S}$ and $v_T \in V_{E_T}$ occupy the same structural position
($\psi_{E_S}(v_S) = \psi_{E_T}(v_T) = s$), both are active at their respective cutoffs,
and $E_S$ has progressed beyond $E_T$'s current stage, then:
\[
\mathrm{TV}(\widehat{P}_s^{E_S}, P_s^T) \leq \alpha_s^{\mathrm{tr}}
\]
for some position-specific transfer-error bound $\alpha_s^{\mathrm{tr}} > 0$.
\end{assumption}

\begin{theorem}[Surface Non-Identifiability]
\label{thm-appdx:nonid}
Suppose there exist two admissible worlds $w^0, w^1$ such that $\mathcal{O}_D^{\infty}(w^0) = \mathcal{O}_D^{\infty}(w^1)$ (identical surface observations), but for a hidden position $s$: $Z_s(w^0) \neq Z_s(w^1)$, where $Z_s(w) = \mathbbm{1}[s \in \mathcal{S}_T^{\mathrm{true}}(w)]$. Then no surface-level method can identify whether $s$ is active:
\begin{equation}
\inf_{\hat{Z}_s = f(\mathcal{O}_D^{\infty})} \max_{j \in \{0,1\}} \Pr_{w^j}(\hat{Z}_s \neq Z_s(w^j)) \geq \tfrac{1}{2}
\end{equation}
If the two worlds imply separated foresight distributions, $\mathrm{TV}(P_F^{w^0}, P_F^{w^1}) \geq 2\Delta_F$, then:
\begin{equation}
\inf_{\widehat{P}_F = f(\mathcal{O}_D^{\infty})} \max_{j \in \{0,1\}} \mathrm{TV}(\widehat{P}_F, P_F^{w^j}) \geq \Delta_F
\end{equation}
This holds for all $K$ including $K = \infty$: adding more surface-matched analogies cannot eliminate foresight risk caused by $D$-unidentifiable hidden mechanisms.
\end{theorem}

\begin{proof}
Let a possibly randomized surface-level estimator be written as
\[
\hat{Z}_s=\varphi(\mathcal{O}_D^\infty,R),
\]
where $R$ denotes internal randomness independent of the world. Since
\[
\mathcal{O}_D^\infty(w^0)=\mathcal{O}_D^\infty(w^1),
\]
the distribution of $\hat{Z}_s$ is the same under $w^0$ and $w^1$.

Without loss of generality, suppose
\[
Z_s(w^0)=0,
\qquad
Z_s(w^1)=1.
\]
Let
\[
r=\Pr(\hat{Z}_s=1).
\]
Because the estimator has the same output distribution in both worlds, its error probability in world $w^0$ is $r$, while its error probability in world $w^1$ is $1-r$. Hence
\[
\max_{j\in\{0,1\}}
\Pr_{w^j}(\hat{Z}_s\neq Z_s(w^j))
=
\max\{r,1-r\}
\geq
\frac12.
\]
Taking the infimum over all surface-level estimators gives
\[
\inf_{\hat{Z}_s=f(\mathcal{O}_D^\infty)}
\max_{j\in\{0,1\}}
\Pr_{w^j}(\hat{Z}_s\neq Z_s(w^j))
\geq
\frac12.
\]

For the foresight part, let a possibly randomized forecast be
\[
\widehat{P}_F=\Phi(\mathcal{O}_D^\infty,R).
\]
Again, because the full surface observation is identical in $w^0$ and $w^1$, the forecast has the same distribution under both worlds. For any realization of $R$, the triangle inequality for total variation gives
\[
\mathrm{TV}(P_F^{w^0},P_F^{w^1})
\leq
\mathrm{TV}(P_F^{w^0},\widehat{P}_F)
+
\mathrm{TV}(\widehat{P}_F,P_F^{w^1}).
\]
Taking expectation over $R$ if the method is randomized, we obtain
\[
\mathbb{E}_R\mathrm{TV}(P_F^{w^0},\widehat{P}_F)
+
\mathbb{E}_R\mathrm{TV}(\widehat{P}_F,P_F^{w^1})
\geq
\mathrm{TV}(P_F^{w^0},P_F^{w^1}).
\]
Therefore,
\[
\max_{j\in\{0,1\}}
\mathbb{E}_R\mathrm{TV}(\widehat{P}_F,P_F^{w^j})
\geq
\frac12
\mathrm{TV}(P_F^{w^0},P_F^{w^1}).
\]
Using the assumed separation
\[
\mathrm{TV}(P_F^{w^0},P_F^{w^1})\geq 2\Delta_F,
\]
we get
\[
\max_{j\in\{0,1\}}
\mathbb{E}_R\mathrm{TV}(\widehat{P}_F,P_F^{w^j})
\geq
\Delta_F.
\]
For deterministic forecasts, the expectations over $R$ disappear.

Finally, a finite-budget surface matcher with any budget $K$ is only a measurable function of $\mathcal{O}_D^\infty$. Since even the full surface observation is identical in the two worlds, no finite subset of surface-matched analogies can distinguish them. Thus the lower bound holds for all finite $K$, and also when $K=\infty$.
\end{proof}

\begin{lemma}[Position-wise Surface Lower Bound]
\label{lem-appdx:position-lower}
Let $H \subseteq \mathcal{S}_\star$ be a set of candidate hidden positions. Suppose there exist two admissible worlds $w^0,w^1$ such that
\[
\mathcal{O}_D^\infty(w^0)=\mathcal{O}_D^\infty(w^1),
\]
and for every $s\in H$,
\[
\mathrm{TV}(P_s^{w^0},P_s^{w^1})\geq 2\Delta_s.
\]
Define the restricted position-wise risk
\[
\mathcal{R}_H(\mathcal{A},w)
=
\sum_{s\in H}
\mu_s
\mathrm{TV}(\widehat{P}_s^{\mathcal{A}},P_s^w).
\]
Then every surface-level method satisfies
\[
\inf_{\mathcal{A}_D}
\max_{j\in\{0,1\}}
\mathcal{R}_H(\mathcal{A}_D,w^j)
\geq
\sum_{s\in H}
\mu_s\Delta_s.
\]
\end{lemma}
\begin{proof}
Since $\mathcal{O}_D^\infty(w^0)=\mathcal{O}_D^\infty(w^1)$, any surface-level method produces the same position-wise forecast in both worlds. Hence, for every $s\in H$,
\[
\widehat{P}_s^{\mathcal{A}_D}(w^0)
=
\widehat{P}_s^{\mathcal{A}_D}(w^1)
=
\widehat{P}_s^{\mathcal{A}_D}.
\]

By the triangle inequality for total variation,
\[
\mathrm{TV}(P_s^{w^0},P_s^{w^1})
\leq
\mathrm{TV}(P_s^{w^0},\widehat{P}_s^{\mathcal{A}_D})
+
\mathrm{TV}(\widehat{P}_s^{\mathcal{A}_D},P_s^{w^1}).
\]

Using the assumed separation
\[
\mathrm{TV}(P_s^{w^0},P_s^{w^1})\geq 2\Delta_s,
\]
we obtain
\[
\mathrm{TV}(\widehat{P}_s^{\mathcal{A}_D},P_s^{w^0})
+
\mathrm{TV}(\widehat{P}_s^{\mathcal{A}_D},P_s^{w^1})
\geq
2\Delta_s.
\]

Multiplying by $\mu_s$ and summing over $s\in H$ gives
\[
\mathcal{R}_H(\mathcal{A}_D,w^0)
+
\mathcal{R}_H(\mathcal{A}_D,w^1)
\geq
2
\sum_{s\in H}
\mu_s\Delta_s.
\]

Therefore,
\[
\max_{j\in\{0,1\}}
\mathcal{R}_H(\mathcal{A}_D,w^j)
\geq
\frac{1}{2}
\left[
\mathcal{R}_H(\mathcal{A}_D,w^0)
+
\mathcal{R}_H(\mathcal{A}_D,w^1)
\right]
\geq
\sum_{s\in H}
\mu_s\Delta_s.
\]

Taking the infimum over all surface-level methods yields
\[
\inf_{\mathcal{A}_D}
\max_{j\in\{0,1\}}
\mathcal{R}_H(\mathcal{A}_D,w^j)
\geq
\sum_{s\in H}
\mu_s\Delta_s.
\]
\end{proof}

\begin{corollary}[Mechanism Matching Breaks the Lower Bound]
\label{cor-appdx:mechanism}
Let $H\subseteq\mathcal{S}_\star$ be a set of hidden candidate positions satisfying the assumptions of Lemma~\ref{lem-appdx:position-lower}. Suppose further that mechanism-level evidence distinguishes the two worlds:
\[
\mathcal{O}_M^\infty(w^0)\neq \mathcal{O}_M^\infty(w^1),
\]
and that mechanism matching retrieves analogies covering every $s\in H$ with transfer error
\[
\mathrm{TV}(\widehat{P}_s^{\mathcal{A}_M},P_s^{w^j})\leq \alpha_s^{\mathrm{tr}}
\]
for all $s\in H$ and $j\in\{0,1\}$. Then
\[
\max_j \mathcal{R}_H(\mathcal{A}_M,w^j)
\leq
\sum_{s\in H}\mu_s\alpha_s^{\mathrm{tr}}.
\]
If
\[
\sum_{s\in H}\mu_s\alpha_s^{\mathrm{tr}}
<
\sum_{s\in H}\mu_s\Delta_s,
\]
then mechanism matching strictly dominates every surface method on hidden positions:
\[
\inf_{\mathcal{A}_M}
\max_j
\mathcal{R}_H(\mathcal{A}_M,w^j)
<
\inf_{\mathcal{A}_D}
\max_j
\mathcal{R}_H(\mathcal{A}_D,w^j).
\]
\end{corollary}
\begin{proof}
By the mechanism-transfer condition, for every $s\in H$ and every $j\in\{0,1\}$,
\[
\mathrm{TV}(\widehat{P}_s^{\mathcal{A}_M},P_s^{w^j})
\leq
\alpha_s^{\mathrm{tr}}.
\]
Therefore,
\[
\mathcal{R}_H(\mathcal{A}_M,w^j)
=
\sum_{s\in H}
\mu_s
\mathrm{TV}(\widehat{P}_s^{\mathcal{A}_M},P_s^{w^j})
\leq
\sum_{s\in H}
\mu_s
\alpha_s^{\mathrm{tr}}.
\]
Taking the maximum over $j$ preserves the same upper bound:
\[
\max_j
\mathcal{R}_H(\mathcal{A}_M,w^j)
\leq
\sum_{s\in H}
\mu_s
\alpha_s^{\mathrm{tr}}.
\]

By Lemma~\ref{lem-appdx:position-lower}, every surface-level method satisfies
\[
\inf_{\mathcal{A}_D}
\max_j
\mathcal{R}_H(\mathcal{A}_D,w^j)
\geq
\sum_{s\in H}
\mu_s
\Delta_s.
\]
If
\[
\sum_{s\in H}\mu_s\alpha_s^{\mathrm{tr}}
<
\sum_{s\in H}\mu_s\Delta_s,
\]
then the mechanism-level upper bound lies strictly below the surface-level lower bound. Hence there exists a mechanism-level method whose worst-case hidden-position risk is strictly smaller than the worst-case risk of every surface-level method.
\end{proof}
\begin{theorem}[Cross-Analogy Confirmation]
\label{thm-appdx:recovery}
Let $Z_s = \mathbbm{1}[s \in \mathcal{S}_T^{\mathrm{true}}]$ and
$X_{k,s} = \mathbbm{1}[\text{analogy } E_k \text{ confirms position } s]$.
Let $\pi_s=P(Z_s=1)$ and $X_{\mathcal A,s}=\{X_{k,s}:E_k\in\mathcal A\}$.

Assume $X_{k,s} \mid Z_s$ are conditionally independent with
\[
q_{k,s}=P(X_{k,s}=1 \mid Z_s=1,E_k\in\mathcal A),
\qquad
p_{k,s}=P(X_{k,s}=1 \mid Z_s=0,E_k\in\mathcal A),
\]
where $q_{k,s} > p_{k,s}$. The probabilities are understood conditional on the selection rule that produced $\mathcal A$.

Then:
\[
\log \frac{P(Z_s = 1 \mid X_{\mathcal{A},s})}{P(Z_s = 0 \mid X_{\mathcal{A},s})}
=
\log\frac{\pi_s}{1-\pi_s}
+
\sum_{E_k \in \mathcal{A}}
\left[
X_{k,s} \log\frac{q_{k,s}}{p_{k,s}}
+
(1-X_{k,s}) \log\frac{1-q_{k,s}}{1-p_{k,s}}
\right].
\]
The Bayes factor grows multiplicatively: each independent confirming analogy multiplies odds by $q_{k,s}/p_{k,s}$.

In the homogeneous positive-confirmation case, $q_{k,s}=q_s$, $p_{k,s}=p_s$, and $X_{k,s}=1$ for all $E_k\in\mathcal A$, the number of confirmations needed to reach posterior confidence $1-\delta$ is
\[
n_s^*
=
\left\lceil
\frac{
\log\frac{1-\delta}{\delta}
-
\log\frac{\pi_s}{1-\pi_s}
}{
\log\frac{q_s}{p_s}
}
\right\rceil.
\]
Under the calibrated regime $(\pi_s = 0.5, q_s = 1, p_s = 0.2)$, this gives $n_s^*=2$ for $\delta=0.05$.

For foresight, letting
\[
\delta_s=P(Z_s=0\mid X_{\mathcal A,s}),
\]
the posterior expected risk at position $s$ satisfies
\[
\mathbb{E}[\mathrm{TV}(\widehat{P}_s^{\mathcal{A}_M}, P_s^T) \mid X_{\mathcal{A},s}]
\leq
(1 - \delta_s)\alpha_s^{\mathrm{tr}} + \delta_s.
\]
\end{theorem}

\begin{proof}
Let
\[
X_{\mathcal{A},s}=\{X_{k,s}:E_k\in\mathcal{A}\}.
\]
For a fixed realization $x_{\mathcal{A},s}$, define the likelihood under the structural hypothesis $Z_s=1$:
\[
L_1(x_{\mathcal{A},s})
=
P(X_{\mathcal{A},s}=x_{\mathcal{A},s}\mid Z_s=1).
\]
By conditional independence,
\[
L_1(x_{\mathcal{A},s})
=
\prod_{E_k\in\mathcal{A}}
q_{k,s}^{x_{k,s}}
(1-q_{k,s})^{1-x_{k,s}}.
\]
Similarly, under the coincidence hypothesis $Z_s=0$,
\[
L_0(x_{\mathcal{A},s})
=
P(X_{\mathcal{A},s}=x_{\mathcal{A},s}\mid Z_s=0)
=
\prod_{E_k\in\mathcal{A}}
p_{k,s}^{x_{k,s}}
(1-p_{k,s})^{1-x_{k,s}}.
\]

By Bayes' theorem,
\[
P(Z_s=1\mid X_{\mathcal{A},s})
=
\frac{
\pi_s L_1(X_{\mathcal{A},s})
}{
\pi_s L_1(X_{\mathcal{A},s})
+
(1-\pi_s)L_0(X_{\mathcal{A},s})
}.
\]
Likewise,
\[
P(Z_s=0\mid X_{\mathcal{A},s})
=
\frac{
(1-\pi_s) L_0(X_{\mathcal{A},s})
}{
\pi_s L_1(X_{\mathcal{A},s})
+
(1-\pi_s)L_0(X_{\mathcal{A},s})
}.
\]
Therefore the posterior odds are
\[
\frac{
P(Z_s=1\mid X_{\mathcal{A},s})
}{
P(Z_s=0\mid X_{\mathcal{A},s})
}
=
\frac{\pi_s}{1-\pi_s}
\cdot
\frac{L_1(X_{\mathcal{A},s})}{L_0(X_{\mathcal{A},s})}.
\]
The likelihood ratio is
\[
\frac{L_1(X_{\mathcal{A},s})}{L_0(X_{\mathcal{A},s})}
=
\prod_{E_k\in\mathcal{A}}
\left(\frac{q_{k,s}}{p_{k,s}}\right)^{X_{k,s}}
\left(\frac{1-q_{k,s}}{1-p_{k,s}}\right)^{1-X_{k,s}}.
\]
Taking logarithms gives
\[
\log \frac{P(Z_s = 1 \mid X_{\mathcal{A},s})}{P(Z_s = 0 \mid X_{\mathcal{A},s})}
=
\log\frac{\pi_s}{1-\pi_s}
+
\sum_{E_k\in \mathcal{A}}
\left[
X_{k,s}\log\frac{q_{k,s}}{p_{k,s}}
+
(1-X_{k,s})\log\frac{1-q_{k,s}}{1-p_{k,s}}
\right].
\]

If $X_{k,s}=1$ for every $E_k\in\mathcal{A}$, then the Bayes factor reduces to
\[
BF_s(\mathcal{A})
=
\prod_{E_k\in\mathcal{A}}
\frac{q_{k,s}}{p_{k,s}}.
\]
Thus every independent confirming analogy multiplies the posterior odds by $q_{k,s}/p_{k,s}$.

In the homogeneous case $q_{k,s}=q_s$ and $p_{k,s}=p_s$, with $K$ positive confirmations,
\[
BF_s(K)
=
\left(\frac{q_s}{p_s}\right)^K.
\]
The posterior odds are
\[
\frac{
P(Z_s=1\mid K\text{ positives})
}{
P(Z_s=0\mid K\text{ positives})
}
=
\frac{\pi_s}{1-\pi_s}
\left(\frac{q_s}{p_s}\right)^K.
\]
To achieve posterior confidence at least $1-\delta$, we require
\[
\frac{\pi_s}{1-\pi_s}
\left(\frac{q_s}{p_s}\right)^K
\geq
\frac{1-\delta}{\delta}.
\]
Solving for $K$ yields
\[
K
\geq
\frac{
\log\frac{1-\delta}{\delta}
-
\log\frac{\pi_s}{1-\pi_s}
}{
\log\frac{q_s}{p_s}
}.
\]
Hence the minimum integer number of confirmations is
\[
n_s^*
=
\left\lceil
\frac{
\log\frac{1-\delta}{\delta}
-
\log\frac{\pi_s}{1-\pi_s}
}{
\log\frac{q_s}{p_s}
}
\right\rceil.
\]
For the calibrated regime
\[
\pi_s=0.5,\qquad q_s=1,\qquad p_s=0.2,\qquad \delta=0.05,
\]
we have prior odds equal to $1$, and the required posterior odds are
\[
\frac{1-\delta}{\delta}
=
\frac{0.95}{0.05}
=
19.
\]
Since
\[
\frac{q_s}{p_s}=5,
\]
one confirmation gives odds $5<19$, while two confirmations give odds $25>19$. Therefore
\[
n_s^*=2.
\]

For the foresight-risk bound, define
\[
\delta_s
=
P(Z_s=0\mid X_{\mathcal{A},s}).
\]
With posterior probability $1-\delta_s$, position $s$ is truly structural. Under this event, the mechanism-transfer assumption gives
\[
\mathrm{TV}(\widehat{P}_s^{\mathcal{A}_M},P_s^T)
\leq
\alpha_s^{\mathrm{tr}}.
\]
With posterior probability $\delta_s$, the inferred position is coincidental or otherwise not truly structural. In that case we use the trivial upper bound
\[
\mathrm{TV}(\widehat{P}_s^{\mathcal{A}_M},P_s^T)
\leq
1.
\]
Therefore,
\[
\mathbb{E}
\left[
\mathrm{TV}(\widehat{P}_s^{\mathcal{A}_M},P_s^T)
\mid X_{\mathcal{A},s}
\right]
\leq
(1-\delta_s)\alpha_s^{\mathrm{tr}}
+
\delta_s\cdot 1.
\]
Thus,
\[
\mathbb{E}
\left[
\mathrm{TV}(\widehat{P}_s^{\mathcal{A}_M},P_s^T)
\mid X_{\mathcal{A},s}
\right]
\leq
(1-\delta_s)\alpha_s^{\mathrm{tr}}
+
\delta_s.
\]
This separates the two sources of foresight error: identification error $\delta_s$ and transfer error $\alpha_s^{\mathrm{tr}}$.
\end{proof}

\section{More Details of \adrb}
\label{appdx:adrb}

We list the rubric prompts used in \adrb as below.
\begin{tcolorbox}[title=Call 2 -- Per-Analogy Scoring (AVS) \& Cross-Analogy Reasoning (CARS) prompt, breakable]
\small
\setlength{\tabcolsep}{0.1mm}{
\textcolor{pink-color}{\textless{}task\textgreater{}}\\
You are evaluating the quality of analogical reasoning in a report. This covers BOTH individual analogy quality AND cross-analogy synthesis.\\
\\
TARGET EVENT: \{event\_name\} (as of \{cutoff\_date\})\\
ANALOGIES THE SYSTEM FOUND: \{extracted\_analogies\}\\
REFERENCE SET (expert-curated analogies from academic literature): \{reference\_set\}\\
\textcolor{pink-color}{\textless{}/task\textgreater{}}\\
\textcolor{pink-color}{\textless{}part-A: per-analogy scoring (AVS)\textgreater{}}\\
For each analogy in the system's list, evaluate:\\
\\
\textbf{1. factual\_existence:} ``PASS'' if real historical event, ``FAIL'' if hallucinated.\\
\\
\textbf{2. structural\_relevance (0--4)} --- how complete is the causal mechanism mapping?\\
\hspace*{2mm}0 = No structural relevance; pure surface match (same country/topic only).\\
\hspace*{2mm}1 = Shares one superficial feature (``both involve banks'').\\
\hspace*{2mm}2 = Identifies 1--2 genuine shared mechanisms but without causal chain depth.\\
\hspace*{2mm}3 = Identifies 3+ shared mechanisms with causal chain articulation (trigger $\rightarrow$ mediator $\rightarrow$ outcome).\\
\hspace*{2mm}4 = Complete structural mapping: preconditions, causal chains, mechanisms, and outcomes all explicitly aligned between source and target. Must name specific institutional actors or variables in BOTH events.\\
\\
\textit{Calibration} --- Score 4: ``Trust companies (1907) operated outside the National Banking System with $\sim$5\% reserves vs 25\% for national banks. Today's SIVs/conduits similarly operate outside FDIC coverage and bank capital requirements. Both were funded by demandable/rollable short-term liabilities (deposits then, ABCP now) against illiquid long-term assets. The structural role is identical: maturity-transforming intermediary outside the safety net.''\\
\textit{Calibration} --- Score 2: ``Like the S\&L crisis, this involves real estate losses leading to financial institution failures.''\\
\\
\textbf{3. is\_in\_reference\_set:} ``YES'' if matches a reference event, ``PARTIAL'' if sibling/related, ``NO'' otherwise.\\
\\
\textbf{4. novelty:} ``NOT\_NOVEL'' if in reference set; ``NOVEL\_VALID'' if real event with structural\_relevance $\geq$ 2 not in reference set; ``NOVEL\_INVALID'' otherwise.\\
\\
\textbf{5. difference\_awareness (0--2):}\\
\hspace*{2mm}0 = No limitations discussed.\\
\hspace*{2mm}1 = Mentions differences superficially (``the scale is different'').\\
\hspace*{2mm}2 = Identifies WHERE the analogy breaks down and WHY it matters for the analysis.\\
\textcolor{pink-color}{\textless{}/part-A\textgreater{}}\\
\textcolor{pink-color}{\textless{}part-B: cross-analogy reasoning (CARS)\textgreater{}}\\
Evaluate how the report uses multiple analogies TOGETHER. This is the critical test --- most reports list analogies pairwise but never synthesize across them.\\
\\
\textbf{B1. Invariant Extraction Quality (0--4).} Sub-classify each cross-analogy invariant the report identifies:\\
\hspace*{2mm}\textbf{L3-D (descriptive):} ``All three events share feature X'' --- shared attribute without specifying STRUCTURAL ROLE (trigger? amplifier? enabler? mediator?). Score weight: 0.4 per invariant.\\
\hspace*{2mm}\textbf{L3-S (structural):} X occupies the SAME STRUCTURAL ROLE in all events' causal chains, even when instantiated by different domain-specific variables. Score weight: 1.0 per invariant.\\
\\
\textit{Calibration} --- L3-S: ``Across 1907 (Heinze director interlocks), Japan 1990 (keiretsu cross-shareholdings), LTCM 1998 (OTC counterparty opacity), an opaque interconnection network serves as the AMPLIFIER that transforms localized shock into systemic crisis.''\\
\textit{Calibration} --- L3-D: ``In every case --- 1907, 1929, 2008 --- institutions built up huge leverage on illiquid assets.''\\
\\
\begin{tabular}{cp{0.78\linewidth}}
\toprule
Score & Criteria \\
\midrule
0 & No cross-analogy comparison attempted \\
1 & Notes shared features but no systematic analysis \\
2 & Identifies 1--2 invariants, mainly L3-D \\
3 & Identifies 3+ invariants, at least 1 is L3-S \\
4 & Systematic extraction + explicit universal vs domain-specific distinction + $\geq$2 L3-S invariants \\
\bottomrule
\end{tabular}\\
\\
\textbf{B2. Hidden Factor Inference (0--3).} Does the report infer SPECIFIC hidden factors in the TARGET EVENT based on patterns across multiple analogies?\\
\\
\begin{tabular}{cp{0.78\linewidth}}
\toprule
Score & Criteria \\
\midrule
0 & No inference of unobserved factors \\
1 & Generic statement (``there may be hidden risks'') --- L4-mention \\
2 & Names specific hidden factor + $\geq$1 analogy justification, but without completing cross-analogy derivation --- L4-partial \\
3 & Complete derivation: (a) names structural role recurring across analogies, (b) cites $\geq$2 analogies with specific instantiations, (c) infers specific hidden factor in target, (d) explains why it hasn't been identified yet --- L4-full \\
\bottomrule
\end{tabular}\\
\\
\textbf{B3. Analogy Differentiation (0--2).}\\
\\
\begin{tabular}{cp{0.78\linewidth}}
\toprule
Score & Criteria \\
\midrule
0 & Treats all analogies as interchangeable supporting examples \\
1 & Notes different analogies highlight different aspects \\
2 & Explicitly maps WHICH insight comes from WHICH analogy and WHY certain analogies are more informative for specific aspects \\
\bottomrule
\end{tabular}\\
\textcolor{pink-color}{\textless{}/part-B\textgreater{}}\\
\textcolor{pink-color}{\textless{}output\textgreater{}}\\
Output exactly ONE JSON object:\\
\texttt{\{}\\
\hspace*{2mm}\texttt{"per\_analogy": [\{"event\_name": "...", "factual\_existence": "PASS|FAIL",}\\
\hspace*{4mm}\texttt{"structural\_relevance": 0, "is\_in\_reference\_set": "YES|PARTIAL|NO",}\\
\hspace*{4mm}\texttt{"novelty": "NOT\_NOVEL|NOVEL\_VALID|NOVEL\_INVALID",}\\
\hspace*{4mm}\texttt{"difference\_awareness": 0\}],}\\
\hspace*{2mm}\texttt{"cross\_analogy\_reasoning": \{}\\
\hspace*{4mm}\texttt{"B1\_invariant\_extraction": \{"score": 0, "L3\_S\_count": 0, "L3\_D\_count": 0, "justification": "..."\},}\\
\hspace*{4mm}\texttt{"B2\_hidden\_factor\_inference": \{"score": 0, "L4\_claims": [...], "justification": "..."\},}\\
\hspace*{4mm}\texttt{"B3\_analogy\_differentiation": \{"score": 0, "justification": "..."\},}\\
\hspace*{4mm}\texttt{"CARS": (B1+B2+B3)/9 \},}\\
\hspace*{2mm}\texttt{"aggregate": \{ "structural\_precision": 0.0, "reference\_recall": 0.0,}\\
\hspace*{4mm}\texttt{"cross\_domain\_rate": 0.0, "novel\_valid\_count": 0,}\\
\hspace*{4mm}\texttt{"mean\_structural\_relevance": 0.0,}\\
\hspace*{4mm}\texttt{"mean\_AVS": mean of (struct\_rel + diff\_aware + novelty\_bonus)/10,}\\
\hspace*{4mm}\texttt{"ARQ": mean\_AVS$\times$0.6 + CARS$\times$0.4 \} \}}\\
\textcolor{pink-color}{\textless{}/output\textgreater{}}
}
\end{tcolorbox}

\begin{tcolorbox}[title=Call 3 -- Atomic Claim Decomposition (L1--L4) \& Hidden-Factor Detection prompt, breakable]
\small
\setlength{\tabcolsep}{0.1mm}{
\textcolor{pink-color}{\textless{}task\textgreater{}}\\
Decompose this analysis report into atomic analytical claims. Each claim should be a single assertion that can be independently evaluated. For each claim, classify its level.\\
\textcolor{pink-color}{\textless{}/task\textgreater{}}\\
\textcolor{pink-color}{\textless{}rubric\textgreater{}}\\
\textbf{L1 (Surface Fact):} A factual observation about the current situation, available from basic news/data. No historical reasoning involved. Also includes name-dropping a historical event WITHOUT mechanism mapping.\\
\hspace*{2mm}\textit{Example:} ``Subprime mortgage delinquencies have been rising since late 2006.''\\
\hspace*{2mm}\textit{Example (name-drop):} ``This is reminiscent of the 2008 financial crisis.'' [No mechanism mapping $\rightarrow$ L1]\\
\\
\textbf{L2 (Single-Analogy Insight):} An insight that explicitly draws on ONE specific historical event and maps a CONCRETE MECHANISM from that event to the current situation. Must name what structural element is shared and how it operates in both contexts.\\
\hspace*{2mm}\textit{Example:} ``Like the 1907 trust companies, today's SIVs operate outside the lender-of-last-resort perimeter, making them vulnerable to wholesale funding runs.'' [Specific mechanism: LoLR perimeter gap $\rightarrow$ run vulnerability]\\
\\
\textbf{L3-D (Cross-Analogy Descriptive Pattern):} Notes that multiple ($\geq$2) historical events share a common FEATURE, but without specifying the STRUCTURAL ROLE that feature plays.\\
\hspace*{2mm}\textit{Example:} ``In every major crisis --- 1907, 1929, 2008 --- excess leverage was present.'' [Shared feature, but is leverage a trigger, amplifier, or enabler?]\\
\\
\textbf{L3-S (Cross-Analogy Structural Pattern):} Identifies that multiple ($\geq$2) historical events share a mechanism occupying the SAME STRUCTURAL ROLE in the causal chain, even when instantiated by different domain-specific variables. Must explicitly name or clearly imply the structural role.\\
\hspace*{2mm}\textit{Example:} ``Across 1907 (Heinze director interlocks), Japan 1990 (keiretsu cross-shareholdings), and LTCM 1998 (OTC counterparty opacity), an opaque interconnection network serves as the AMPLIFIER that transforms a localized shock into systemic crisis.''\\
\\
\textbf{L4 (Hidden Factor Inference):} An inference about a SPECIFIC, CURRENTLY UNRECOGNIZED factor in the current situation, justified by cross-analogy evidence from $\geq$2 historical events. ALL of the following must be present:\\
\hspace*{2mm}(a) Names a SPECIFIC factor or mechanism (not ``there may be hidden risks'')\\
\hspace*{2mm}(b) Explains WHY this factor is likely present, citing $\geq$2 analogies\\
\hspace*{2mm}(c) The factor must be something not yet widely discussed in pre-cutoff reporting\\
\hspace*{2mm}(d) Must be a FACTOR or MECHANISM identification, NOT a probabilistic forecast or scenario prediction\\
\\
\hspace*{2mm}\textit{Valid L4:} ``Based on the recurring pattern of opaque amplification networks in 1907/1998, an analogous hidden interconnection likely exists --- possibly CDS counterparty exposure concentrated in a few dealers, or repo market concentration in clearing banks.''\\
\hspace*{2mm}\textit{NOT L4} (probabilistic forecast $\rightarrow$ L1): ``There is a 70\% probability of recession.''\\
\hspace*{2mm}\textit{NOT L4} (scenario $\rightarrow$ L1): ``Based on historical patterns, a 30--40\% market correction is likely.''\\
\hspace*{2mm}\textit{NOT L4} (too vague $\rightarrow$ L1): ``There may be hidden systemic risks.''\\
\hspace*{2mm}\textit{NOT L4} (known factor $\rightarrow$ L1): ``The ABCP market freeze is causing liquidity problems.''\\
\textcolor{pink-color}{\textless{}/rubric\textgreater{}}\\
\textcolor{pink-color}{\textless{}critical-rules\textgreater{}}\\
1. Probabilistic forecasts and scenario predictions are NEVER L4, regardless of historical references. L4 identifies a HIDDEN FACTOR, not an OUTCOME.\\
2. Name-dropping a historical event without mechanism mapping is L1, not L2.\\
3. Noting a shared feature across events without specifying its structural role is L3-D, not L3-S.\\
4. A claim can only be L3-S or L4 if it references MULTIPLE ($\geq$2) historical events.\\
5. If unsure between L3-D and L3-S, check: does the claim name or clearly imply a structural role (trigger, amplifier, mediator, enabler, outcome)? If yes $\rightarrow$ L3-S. If no $\rightarrow$ L3-D.\\
\textcolor{pink-color}{\textless{}/critical-rules\textgreater{}}\\
\textcolor{pink-color}{\textless{}input\textgreater{}}\\
REPORT: \{report\}\\
HIDDEN FACTORS FROM REFERENCE SET (use to check if any claims --- at ANY level --- match known hidden factors): \{hidden\_factors\}\\
\textcolor{pink-color}{\textless{}/input\textgreater{}}\\
\textcolor{pink-color}{\textless{}output\textgreater{}}\\
\texttt{\{ "claims": [\{}\\
\hspace*{2mm}\texttt{"claim": "exact text from report (may be condensed)",}\\
\hspace*{2mm}\texttt{"level": "L1 | L2 | L3-D | L3-S | L4",}\\
\hspace*{2mm}\texttt{"historical\_events\_referenced": ["event1", "event2"],}\\
\hspace*{2mm}\texttt{"mechanism\_identified": "..." or null,}\\
\hspace*{2mm}\texttt{"structural\_role": "trigger|amplifier|mediator|enabler|outcome" or null,}\\
\hspace*{2mm}\texttt{"hidden\_factor\_hit": "name from reference set" or null \}],}\\
\texttt{"aggregate": \{ "L1": 0, "L2": 0, "L3\_D": 0, "L3\_S": 0, "L4": 0,}\\
\hspace*{2mm}\texttt{"total\_claims": 0,}\\
\hspace*{2mm}\texttt{"HF\_at\_any": "count of distinct HFs hit at ANY level",}\\
\hspace*{2mm}\texttt{"HF\_at\_L4": "count of distinct HFs hit at L4 ONLY",}\\
\hspace*{2mm}\texttt{"HF\_total": "total HFs in reference set" \} \}}\\
\textcolor{pink-color}{\textless{}/output\textgreater{}}
}
\end{tcolorbox}

\begin{tcolorbox}[title=Call 4 -- Reasoning Analysis Score (RAS) prompt, breakable]
\small
\setlength{\tabcolsep}{0.1mm}{
\textcolor{pink-color}{\textless{}task\textgreater{}}\\
You are evaluating the analytical quality of a report about a historical situation. Score on six dimensions.\\
\\
EVENT: \{event\_name\}; TEMPORAL CUTOFF: \{cutoff\_date\}\\
The analyst was asked to write a comprehensive analysis as of the cutoff date.\\
REPORT: \{report\}\\
\\
Score each dimension according to the rubric. Be STRICT and CALIBRATED.\\
\textcolor{pink-color}{\textless{}/task\textgreater{}}\\
\textcolor{pink-color}{\textless{}calibration-guidance\textgreater{}}\\
- Scores of 3--4 (on 0--4 scales) or 3 (on 0--3 scales) should be RARE. They require genuinely exceptional analysis that goes beyond competent reporting.\\
- A well-written report by a frontier LLM that describes mechanisms competently but without historical grounding or structural novelty should typically score 2--3 on D1 and 1--2 on D4.\\
- The TOP score on each dimension requires something a generic deep-research report would NOT naturally produce --- typically historical grounding, structural abstraction, or calibrated uncertainty from precedent.\\
\textcolor{pink-color}{\textless{}/calibration-guidance\textgreater{}}\\
\textcolor{pink-color}{\textless{}D1: causal depth (0--4, weight $\times$2)\textgreater{}}\\
0: Purely descriptive; no causal claims.\\
1: Identifies causes but only immediate/proximate.\\
2: Identifies causal chains with 2+ steps (A$\rightarrow$B$\rightarrow$C); names some structural features.\\
3: Multi-level causal analysis: distinguishes proximate causes from structural causes and enabling conditions; identifies feedback loops and amplification mechanisms.\\
4: Everything in 3, PLUS: explicitly names the structural role each factor plays (trigger, amplifier, mediator, enabler); grounds the causal framework in historical precedent showing WHY these specific mechanisms should be expected; distinguishes this event's causal structure from superficially similar but structurally different events.\\
\\
\textit{Calibration} --- Score 3: ``The crisis operates through a self-reinforcing funding liquidity spiral: asset sales depress prices, triggering margin calls, forcing more sales.'' [Identifies feedback loop --- but no historical grounding or structural role labeling]\\
\textit{Calibration} --- Score 4: ``Three feedback loops operate simultaneously. Loop 1 (mark-to-market $\rightarrow$ haircut): ABX prices fall $\rightarrow$ dealers mark books $\rightarrow$ repo lenders widen haircuts $\rightarrow$ forced sales push ABX lower. This is structurally identical to the 1907 trust-company run mechanism, where opacity about which trusts held Heinze-linked assets triggered indiscriminate withdrawal. The TRIGGER is subprime losses; the AMPLIFIER is opacity-driven wholesale funding withdrawal; the MEDIATOR is the ABCP/repo conduit system.''\\
\textcolor{pink-color}{\textless{}/D1\textgreater{}}\\
\textcolor{pink-color}{\textless{}D2: mechanism specificity (0--3, weight $\times$2)\textgreater{}}\\
0: No specific mechanisms named.\\
1: Names mechanisms generically (``contagion'', ``panic'', ``leverage'', ``deleveraging'').\\
2: Specifies concrete transmission channels with institutional detail.\\
3: Specifies channels + quantitative data + named institutional actors.\\
\\
\textit{Calibration} --- Score 2: ``Banks are tightening lending to each other, as shown by widening LIBOR-OIS spreads.''\\
\textit{Calibration} --- Score 3: ``The effective federal funds rate printed at 5.41\% on Aug 9, 16bp above the 5.25\% target --- the largest deviation in over a year. The ECB allotted \texteuro94.8B to 49 banks at 4.00\% full allotment.''\\
\textcolor{pink-color}{\textless{}/D2\textgreater{}}\\
\textcolor{pink-color}{\textless{}D3: scope and coverage (0--3, weight $\times$2)\textgreater{}}\\
0: Covers only one dimension of the event.\\
1: Covers multiple dimensions but misses major ones.\\
2: Comprehensive coverage of known risk dimensions (housing, securitization, shadow banking, interbank, policy response).\\
3: Comprehensive + identifies under-discussed or emerging dimensions a standard analysis would miss (e.g., quant fund deleveraging as second-order effect; European bank dollar-funding dependence; MMF structural vulnerability; commercial real estate as next domino).\\
\textcolor{pink-color}{\textless{}/D3\textgreater{}}\\
\textcolor{pink-color}{\textless{}D4: nuance and uncertainty (0--3, weight $\times$1)\textgreater{}}\\
0: Entirely confident assertions with no hedging.\\
1: Some uncertainty language but generic.\\
2: Specific uncertainty about specific claims with reasoning.\\
3: Distinguishes high- from low-confidence claims; CALIBRATES uncertainty using specific historical precedent; identifies what specific information would change the assessment.\\
\\
\textit{Calibration} --- Score 2: ``The outcome is uncertain and depends on policy responses and the depth of housing correction.''\\
\textit{Calibration} --- Score 3: ``If LIBOR-OIS stays above 100bp, escalate to 1907 trajectory and expect multi-month dislocation; if ABCP issuance resumes within a month, the 1998 template applies --- sharp dislocation, swift resolution. The 1907/1998 distinction hinges on whether the problem is solvency (1907) or pure liquidity (1998).''\\
\textcolor{pink-color}{\textless{}/D4\textgreater{}}\\
\textcolor{pink-color}{\textless{}D5: internal consistency (0--2, weight $\times$1)\textgreater{}}\\
0: Contains contradictions or logically incompatible claims.\\
1: No contradictions but claims are loosely connected; analysis reads as a list of observations.\\
2: Claims form a coherent, mutually reinforcing analytical framework with a clear through-line.\\
\textcolor{pink-color}{\textless{}/D5\textgreater{}}\\
\textcolor{pink-color}{\textless{}D6: actionability (0--2, weight $\times$1)\textgreater{}}\\
0: Analysis is purely retrospective/academic with no forward-looking implications.\\
1: Identifies what to monitor but not specific thresholds or decision triggers.\\
2: Provides specific indicators with thresholds, scenario-conditional recommendations, decision-relevant framings (e.g., monitoring table with current value $\rightarrow$ escalation threshold $\rightarrow$ normalization threshold for LIBOR-OIS, ABCP outstanding, ABX.HE.AAA, etc.).\\
\textcolor{pink-color}{\textless{}/D6\textgreater{}}\\
\textcolor{pink-color}{\textless{}output\textgreater{}}\\
\texttt{\{ "D1\_causal\_depth": \{"score": 0, "justification": "..."\},}\\
\hspace*{2mm}\texttt{"D2\_mechanism\_specificity": \{"score": 0, "justification": "..."\},}\\
\hspace*{2mm}\texttt{"D3\_scope": \{"score": 0, "justification": "..."\},}\\
\hspace*{2mm}\texttt{"D4\_nuance": \{"score": 0, "justification": "..."\},}\\
\hspace*{2mm}\texttt{"D5\_consistency": \{"score": 0, "justification": "..."\},}\\
\hspace*{2mm}\texttt{"D6\_actionability": \{"score": 0, "justification": "..."\},}\\
\hspace*{2mm}\texttt{"RAS": (D1$\times$2 + D2$\times$2 + D3$\times$2 + D4 + D5 + D6) / 27 \}}\\
\textcolor{pink-color}{\textless{}/output\textgreater{}}
}
\end{tcolorbox}

\begin{tcolorbox}[title=Call 5 -- Foresight Quality Score (FQS) prompt, breakable]
\small
\setlength{\tabcolsep}{0.1mm}{
\textcolor{pink-color}{\textless{}task\textgreater{}}\\
You are evaluating the foresight quality of an analysis report written about a historical situation AS OF a specific cutoff date. You know what actually happened after the cutoff --- use this knowledge to evaluate the report's forward-looking claims.\\
\\
EVENT: \{event\_name\}; TEMPORAL CUTOFF: \{cutoff\_date\}\\
The analyst was asked to provide forward-looking assessment as of the cutoff date.\\
REPORT: \{report\}\\
KNOWN HIDDEN FACTORS (factors not widely recognized at cutoff but proved critical): \{hidden\_factors\}\\
\textcolor{pink-color}{\textless{}/task\textgreater{}}\\
\textcolor{pink-color}{\textless{}step-1: extract\textgreater{}}\\
Extract every claim that makes a prediction, projects a scenario, or assesses future risk. Include conditional statements (``if X, then Y''), probability assessments, and scenario descriptions. Be exhaustive.\\
\textcolor{pink-color}{\textless{}/step-1\textgreater{}}\\
\textcolor{pink-color}{\textless{}step-2: score each claim\textgreater{}}\\
\textbf{A. OUTCOME} (verified against what actually happened):\\
\hspace*{2mm}$+1.0$ Confirmed: development occurred substantially as predicted.\\
\hspace*{2mm}$+0.5$ Partially confirmed: directionally correct but details differ materially.\\
\hspace*{2mm}$+0.0$ Unconfirmed: did not occur (but not contradicted).\\
\hspace*{2mm}$-0.3$ Contradicted: opposite occurred or fundamentally wrong mechanism.\\
\\
\textbf{B. GROUNDING (0--3) --- BE STRICT.} This is the key discriminating dimension.\\
\hspace*{2mm}0: No historical analogy or evidence cited. Pure assertion or extrapolation. Includes ``historically, markets recover'' or ``past crises suggest...'' without naming a SPECIFIC event.\\
\hspace*{2mm}1: Names a specific historical event but provides NO mechanism mapping (``reminiscent of 1998''). The event is decoration, not analytical input.\\
\hspace*{2mm}2: Names a specific event AND aligns ONE structural element (trigger OR mechanism OR outcome --- but not all). Example: ``Following the 1907 template, trust-company-style runs on shadow banks are likely.''\\
\hspace*{2mm}3: Names a specific event AND aligns $\geq$2 structural roles AND names the corresponding variable in the CURRENT situation. Example: ``Based on 1907 (trust companies outside NYCH = TRIGGER: opacity-driven run), the current TRIGGER is ABCP conduit opacity; the AMPLIFIER is the same (maturity mismatch + wholesale funding); the likely OUTCOME, per the 1907 template, is sequential institutional failures until a coordinator (Morgan then, Fed now) provides selective support.''\\
\\
\textit{Calibration note:} Generic references to ``history shows'' or ``past crises'' with no specific event = Grounding 0. Naming ``1998'' without saying WHAT about 1998 applies = Grounding 1. Most foresight claims by non-synthesis agents will be Grounding 0 or 1.\\
\\
\textbf{C. SPECIFICITY (1--3):}\\
\hspace*{2mm}1: Vague/directional (``financial system will face significant stress'').\\
\hspace*{2mm}2: Names a causal pathway (``credit contraction via bank deleveraging will slow GDP growth'').\\
\hspace*{2mm}3: Names concrete entities, quantities, or timing with mechanism (``investment banks with subprime MBS exposure $>$10$\times$ capital face solvency risk within 6 months; the Bear Stearns funds are the template'').\\
\\
\textbf{D. DERIVABILITY (0--3):}\\
\hspace*{2mm}0: Impossible without post-cutoff knowledge $\rightarrow$ temporal violation.\\
\hspace*{2mm}1: Unlikely derivable from pre-cutoff evidence alone.\\
\hspace*{2mm}2: Plausible stretch from pre-cutoff evidence + analogical reasoning.\\
\hspace*{2mm}3: Clearly derivable from pre-cutoff evidence + analogies.\\
\\
\textbf{E. TEMPORAL COMPLIANCE:}\\
\hspace*{2mm}\textsc{Clean}: Based on pre-cutoff facts or analogical reasoning (citing full history of analogy events is allowed).\\
\hspace*{2mm}\textsc{Soft\_violation}: References post-cutoff trends about the TARGET event without naming specific post-cutoff events.\\
\hspace*{2mm}\textsc{Hard\_violation}: Explicitly references specific post-cutoff events, data, or outcomes about the TARGET event.\\
\\
\textbf{F. HIDDEN FACTOR HIT:} Does this claim identify or predict any of the known hidden factors listed above? A hit requires the claim to identify the FACTOR, not just predict an outcome that happens to relate to it.\\
\textcolor{pink-color}{\textless{}/step-2\textgreater{}}\\
\textcolor{pink-color}{\textless{}step-3: per-claim FQS\textgreater{}}\\
$w = 1.0$ if \textsc{Clean}, $0.5$ if \textsc{Soft\_violation}, $0.0$ if \textsc{Hard\_violation}.\\
$\mathrm{RQ} = (\mathrm{Grounding} + \mathrm{Specificity} + \mathrm{Derivability}) / 9$\\
$\mathrm{FQS} = w \times (\mathrm{Outcome} \times 0.5 + \mathrm{RQ} \times 0.5)$\\
\textcolor{pink-color}{\textless{}/step-3\textgreater{}}\\
\textcolor{pink-color}{\textless{}step-4: aggregate \& output\textgreater{}}\\
\texttt{\{ "foresight\_claims": [\{}\\
\hspace*{2mm}\texttt{"claim": "...", "outcome": \{"rating": "confirmed|partial|unconfirmed|contradicted",}\\
\hspace*{4mm}\texttt{"value": 0.0, "what\_actually\_happened": "..."\},}\\
\hspace*{2mm}\texttt{"grounding": \{"score": 0, "analogy\_cited": "..." or null,}\\
\hspace*{4mm}\texttt{"structural\_roles\_mapped": ["trigger","mechanism","outcome"] or [],}\\
\hspace*{4mm}\texttt{"justification": "..."\},}\\
\hspace*{2mm}\texttt{"specificity": \{"score": 0, "justification": "..."\},}\\
\hspace*{2mm}\texttt{"derivability": \{"score": 0, "justification": "..."\},}\\
\hspace*{2mm}\texttt{"temporal\_compliance": "CLEAN|SOFT\_VIOLATION|HARD\_VIOLATION",}\\
\hspace*{2mm}\texttt{"hidden\_factor\_hit": "..." or null, "RQ": 0.0, "FQS": 0.0 \}],}\\
\texttt{"aggregate": \{ "total\_foresight\_claims": 0, "FQS\_mean": 0.0,}\\
\hspace*{2mm}\texttt{"FQS\_depth": "mean FQS over claims with Grounding $\geq$ 2 AND Specificity $\geq$ 2",}\\
\hspace*{2mm}\texttt{"FQS\_depth\_n": "count meeting that threshold",}\\
\hspace*{2mm}\texttt{"temporal\_compliance\_rate": 0.0, "confirmed\_rate": 0.0,}\\
\hspace*{2mm}\texttt{"grounded\_rate": "fraction with Grounding $\geq$ 2",}\\
\hspace*{2mm}\texttt{"hidden\_factor\_hits": 0, "hidden\_factor\_total": 0 \} \}}\\
\textcolor{pink-color}{\textless{}/step-4\textgreater{}}
}
\end{tcolorbox}

\section{More Details of \ours}
\label{appdx:method}

\subsection{\ours for historical analogy generation}
\label{appdx:adr_gen}
The algorithmic description of \ours used for 
\begin{algorithm}[H]
        \caption{The \ours Framework}
        \label{alg:crha}
        \begin{algorithmic}[1]
        \STATE \textbf{Required:} Target event $E_T$; candidate analogy universe $\mathcal{U}$; LLM $\mathcal{M}$; representation extractor $\textsc{Repr}(\cdot)$ that produces $(D(E), M(E))$; pattern abstractor $\textsc{Abs}(\cdot)$ that produces the abstract causal pattern $\mathfrak{P}$ and its structural positions $\mathcal{S}_{\mathfrak{P}}$; candidate generator $\textsc{Gen}(\cdot)$; self-analogy filter $\textsc{Filter}(\cdot,\cdot)$; per-position alignment scorer $\textsc{Align}(\cdot,\cdot)$ that produces $\{X_{k,s}\}_{s \in \mathcal{S}_{\mathfrak{P}}}$ for analogy $E_k$; aggregate decider $\textsc{Decide}(\cdot)$; max reflection rounds $R$; candidates per round $K$; carryover size $M$;
        \STATE Extracting the target representations $(D(E_T), M(E_T)) \gets \textsc{Repr}(\mathcal{M}, E_T)$;
        \STATE Abstracting the target pattern $(\mathfrak{P}, \mathcal{S}_{\mathfrak{P}}) \gets \textsc{Abs}(\mathcal{M}, D(E_T), M(E_T))$, where $\mathcal{S}_{\mathfrak{P}}$ enumerates the structural positions to be confirmed by analogies;
        \STATE Sampling an initial candidate analogy set $\mathcal{A}^{(0)} \subset \mathcal{U}$ via $\mathcal{A}^{(0)} \gets \textsc{Gen}(\mathcal{M}, \mathfrak{P}, K)$;
        \STATE Initializing the tried-set $\mathcal{T} \gets \mathcal{A}^{(0)}$, the carryover pool $\mathcal{K} \gets \emptyset$, and the latest decision $\mathcal{D}^* \gets \texttt{null}$;
        \FOR{round $t \in [0, 1, \ldots, R]$}
        \STATE Initializing the per-round evaluation set $\mathcal{E}^{(t)} \gets \emptyset$;
        \FOR{each candidate analogy $E_k \in \mathcal{A}^{(t)}$}
        \IF{$\textsc{Filter}(\mathcal{M}, E_T, E_k)$ flags $E_k$ as alias, sub-instance, or proper subset of $E_T$, i.e., $E_k$ collapses to $E_T$ rather than residing in $\mathcal{U} \setminus \{E_T\}$ in the strict sense}
        \STATE Discarding $E_k$;
        \ELSE
        \STATE Extracting $(D(E_k), M(E_k)) \gets \textsc{Repr}(\mathcal{M}, E_k)$ via Wikipedia-grounded retrieval followed by mechanism extraction;
        \STATE Producing the position-level alignment evidence $\{X_{k,s}\}_{s \in \mathcal{S}_{\mathfrak{P}}} \gets \textsc{Align}(\mathcal{M}, (D(E_T), M(E_T)), (D(E_k), M(E_k)))$, where each $X_{k,s}$ records both the confirming match and the residual gap of analogy $E_k$ at position $s$;
        \STATE Updating $\mathcal{E}^{(t)} \gets \mathcal{E}^{(t)} \cup \{(E_k, \{X_{k,s}\}_s)\}$;
        \ENDIF
        \ENDFOR
        \STATE Merging the carryover from the previous round: $\mathcal{E}^{(t)} \gets \mathcal{E}^{(t)} \cup \mathcal{K}$;
        \STATE Computing the aggregate decision $\mathcal{D}^{(t)} \gets \textsc{Decide}(\mathcal{M}, \mathfrak{P}, \mathcal{E}^{(t)})$, which (i) prunes candidates violating five hard rejection rules (alias, sub-event, super-set, generic descriptor, scope-class mismatch), (ii) ranks the eligible analogies by the fidelity of the induced structural mapping $\phi_k$ from $E_k$ to $E_T$, equivalently by the support that $\{X_{k,s}\}_s$ provides for $\widehat{P}_s^{\mathcal{A}}$ across $s \in \mathcal{S}_{\mathfrak{P}}$ under mechanism specificity and scale alignment, and (iii) returns $\{E^{(t)*}, \text{top-}M, \text{is\_sufficient}, \text{common\_gap}, \text{next\_focus}\}$;
        \STATE Updating $\mathcal{D}^* \gets \mathcal{D}^{(t)}$;
        \IF{$\mathcal{D}^{(t)}.\text{is\_sufficient}$ holds, indicating that the target-side overlap $\Omega_{\mathrm{pos}}$ is adequately covered by the chosen analogy}
        \STATE \textbf{break};
        \ENDIF
        \STATE Updating the carryover pool $\mathcal{K} \gets \mathcal{D}^{(t)}.\text{top-}M$;
        \STATE Generating the next-round candidate analogies $\mathcal{A}^{(t+1)} \gets \textsc{Gen}(\mathcal{M}, \mathfrak{P}, \mathcal{D}^{(t)}.\text{common\_gap}, \mathcal{D}^{(t)}.\text{next\_focus}, \mathcal{T})$, conditioned on the under-confirmed positions in $\mathcal{S}_{\mathfrak{P}}$ and excluding any $E_k \in \mathcal{T}$;
        \STATE Removing duplicates: $\mathcal{A}^{(t+1)} \gets \mathcal{A}^{(t+1)} \setminus \mathcal{T}$;
        \IF{$\mathcal{A}^{(t+1)} = \emptyset$}
        \STATE \textbf{break};
        \ENDIF
        \STATE Updating the tried-set $\mathcal{T} \gets \mathcal{T} \cup \mathcal{A}^{(t+1)}$;
        \ENDFOR
        \STATE \textbf{return} the selected analogy $E^* \gets \mathcal{D}^*.E^{(t)*}$ together with its induced structural mapping $\phi^*: \mathcal{S}_{\mathfrak{P}} \to E^*$;
        \end{algorithmic}
\end{algorithm}

We present the prompts we used in \ours:

\begin{tcolorbox}[title=Candidate Generation prompt, breakable]
\small
\setlength{\tabcolsep}{0.1mm}{
\textcolor{pink-color}{\textless{}task\textgreater{}}\\
You are a robot for finding structurally analogous historical events based on causal analysis. Historical Analogy means identifying past events that share similar causal structures with a contemporary event, not merely similar topics or surface features.\\
\\
Given the causal profile of an input event (including its structural pattern, mechanism, and causal chain), output 5 historical events that share similar causal structures. Candidates may come from entirely different domains (e.g., a financial crisis may be structurally analogous to an epidemic) as long as the underlying causal mechanisms are similar. If there is any reflection, please modify the recommended events based on the reflection.\\
\textcolor{pink-color}{\textless{}/task\textgreater{}}\\
\textcolor{pink-color}{\textless{}format\textgreater{}}\\
\textbf{IMPORTANT:} Output ONLY a JSON list of 5 short event names, like \texttt{[``event1'',``event2'',``event3'',``event4'',``event5'']}. No explanations, no descriptions, no numbering. Each name should be a concise historical event name (under 50 characters).\\
\textcolor{pink-color}{\textless{}/format\textgreater{}}\\
\textcolor{pink-color}{\textless{}example\textgreater{}}\\
Input Event:\\
2008 Financial Crisis\\
Structural Pattern: systemic\_risk\_cascade\\
Mechanism: Interconnected financial instruments amplified localized failures into global systemic collapse\\
Causal Chain: Housing bubble + lax regulation $\rightarrow$ subprime mortgage defaults $\rightarrow$ CDO/CDS losses cascade $\rightarrow$ bank insolvency $\rightarrow$ credit freeze $\rightarrow$ global recession\\
Output: [``1997 Asian financial crisis'', ``South Sea Bubble'', ``Long-Term Capital Management collapse'', ``2010 European sovereign debt crisis'', ``1873 Vienna Stock Exchange crash'']\\
\textcolor{pink-color}{\textless{}/example\textgreater{}}\\
\textcolor{pink-color}{\textless{}history\textgreater{}}\\
\{chat\_history\}\\
\textcolor{pink-color}{\textless{}/history\textgreater{}}\\
\textcolor{pink-color}{\textless{}input\textgreater{}}\\
\{input\_type\}: \{input\}\\
Output:\\
\textcolor{pink-color}{\textless{}/input\textgreater{}}
}
\end{tcolorbox}

\begin{tcolorbox}[title=Self-Analogy / Alias / Sub-instance Filter prompt, breakable]
\small
\setlength{\tabcolsep}{0.1mm}{
\textcolor{pink-color}{\textless{}task\textgreater{}}\\
You will judge whether a candidate event should be FILTERED OUT as a self-analogy when proposing it as an analog for an input event.\\
\\
Input event: \{input\_name\}\\
Candidate event: \{candidate\_name\}\\
\textcolor{pink-color}{\textless{}/task\textgreater{}}\\
\textcolor{pink-color}{\textless{}rules\textgreater{}}\\
Return YES (filter out) ONLY if one of these three holds:\\
\\
\textbf{1. ALIAS} --- candidate is a different NAME for the SAME specific historical event.\\
\hspace*{2mm}YES patterns (templated; angle brackets denote slots, not specific events):\\
\hspace*{4mm}-- $\langle$event under codename$\rangle$ $\equiv$ $\langle$same event under descriptive popular name$\rangle$\\
\hspace*{4mm}-- $\langle$event under regional/local name$\rangle$ $\equiv$ $\langle$same event under translated or global name$\rangle$\\
\hspace*{4mm}-- $\langle$event under full institutional suffix$\rangle$ $\equiv$ $\langle$same event under truncated suffix or root noun$\rangle$\\
\hspace*{4mm}-- $\langle$war named by calendar A date$\rangle$ $\equiv$ $\langle$same war named by calendar B date$\rangle$\\
\hspace*{4mm}-- $\langle$crisis named by location$\rangle$ $\equiv$ $\langle$same crisis named by month/year$\rangle$\\
\\
\textbf{2. SUB-INSTANCE of GENERIC input} --- input names a generic CATEGORY, candidate is one specific instance.\\
\hspace*{2mm}YES patterns:\\
\hspace*{4mm}-- input ``$\langle$generic ceremony class$\rangle$'' + candidate ``$\langle$a regional or cultural variant of that ceremony$\rangle$''\\
\hspace*{4mm}-- input ``$\langle$class of recurring institutional events$\rangle$'' + candidate ``$\langle$one dated instance of that class$\rangle$''\\
\hspace*{4mm}-- input ``$\langle$class of accidents/disasters$\rangle$'' + candidate ``$\langle$one named incident of that kind$\rangle$''\\
\\
\textbf{3. PROPER SUBSET} --- input is a CONTAINER event (a multi-year war, movement, era, crisis, or campaign that comprises many sub-events) and the candidate is one named sub-event WITHIN it.\\
\hspace*{2mm}YES patterns (true proper subsets):\\
\hspace*{4mm}-- input ``$\langle$multi-year war$\rangle$'' + candidate ``$\langle$one named battle within that war$\rangle$''\\
\hspace*{4mm}-- input ``$\langle$decade-long geopolitical confrontation$\rangle$'' + candidate ``$\langle$one named crisis episode within that confrontation$\rangle$''\\
\hspace*{4mm}-- input ``$\langle$social movement spanning years$\rangle$'' + candidate ``$\langle$one named march/protest/campaign within that movement$\rangle$''\\
\hspace*{4mm}-- input ``$\langle$multi-stage operation$\rangle$'' + candidate ``$\langle$one named stage of that operation$\rangle$''\\
\\
\textbf{COUNTER-EXAMPLES} --- do NOT filter these (sibling, not subset):\\
\hspace*{2mm}-- input ``$\langle$conference A in a series$\rangle$'' + candidate ``$\langle$conference B in the same series$\rangle$'' $\rightarrow$ NO\\
\hspace*{2mm}-- input ``$\langle$war A$\rangle$'' + candidate ``$\langle$separate war between different parties, similar mechanism$\rangle$'' $\rightarrow$ NO\\
\hspace*{2mm}-- input ``$\langle$war A$\rangle$'' + candidate ``$\langle$war B as a sibling in a numbered series$\rangle$'' $\rightarrow$ NO\\
\\
If unsure between SAME-DOMAIN-SIBLING (NO) and PROPER-SUBSET (YES), ask: \emph{``Is the candidate event part of the input event's timeline / agenda / formal scope?''} --- if no, it is a sibling, not a subset.\\
\\
Return NO for everything else (genuinely different events sharing vocabulary; different events in same domain; cross-domain analogs).\\
\textcolor{pink-color}{\textless{}/rules\textgreater{}}\\
\textcolor{pink-color}{\textless{}reasoning\textgreater{}}\\
Reason step by step BEFORE giving the final answer:\\
1. Identify the specific historical event each name refers to. Are they the same real-world event under different names?\\
2. If different, check: is one a SUB-INSTANCE of a generic category named by the other?\\
3. If different, check: is one a PROPER SUBSET (named sub-event within a container event)?\\
4. State your final decision.\\
\textcolor{pink-color}{\textless{}/reasoning\textgreater{}}\\
\textcolor{pink-color}{\textless{}output\textgreater{}}\\
Reasoning: $\langle$2--3 sentences$\rangle$\\
Answer: YES or NO\\
\textcolor{pink-color}{\textless{}/output\textgreater{}}
}
\end{tcolorbox}

\begin{tcolorbox}[title=Per-Candidate 5-Dimension Match/Gap Evaluation prompt, breakable]
\small
\setlength{\tabcolsep}{0.1mm}{
\textcolor{pink-color}{\textless{}task\textgreater{}}\\
You will compare an input event against a candidate analogy along 5 causal dimensions. For each dimension, identify both what MATCHES and what does NOT match.\\
\textcolor{pink-color}{\textless{}/task\textgreater{}}\\
\textcolor{pink-color}{\textless{}input-event\textgreater{}}\\
Input Event: \{input\_name\}\\
\hspace*{2mm}Preconditions: \{input\_precond\}\\
\hspace*{2mm}Causal Chain: \{input\_chain\}\\
\hspace*{2mm}Mechanism: \{input\_mechanism\}\\
\hspace*{2mm}Outcome \& Impact: \{input\_outcome\}\\
\hspace*{2mm}Structural Pattern: \{input\_pattern\}\\
\textcolor{pink-color}{\textless{}/input-event\textgreater{}}\\
\textcolor{pink-color}{\textless{}candidate-event\textgreater{}}\\
Candidate Event: \{cand\_name\}\\
\hspace*{2mm}Preconditions: \{cand\_precond\}\\
\hspace*{2mm}Causal Chain: \{cand\_chain\}\\
\hspace*{2mm}Mechanism: \{cand\_mechanism\}\\
\hspace*{2mm}Outcome \& Impact: \{cand\_outcome\}\\
\hspace*{2mm}Structural Pattern: \{cand\_pattern\}\\
\textcolor{pink-color}{\textless{}/candidate-event\textgreater{}}\\
\textcolor{pink-color}{\textless{}reasoning\textgreater{}}\\
\textbf{REASONING-FIRST PROCESS} --- for EACH dimension, work through these steps:\\
\hspace*{2mm}Step 1: Read input's value for this dimension and candidate's value.\\
\hspace*{2mm}Step 2: Identify what specifically ALIGNS --- name the shared mechanism / role / dynamic. Be concrete.\\
\hspace*{2mm}Step 3: Identify what specifically does NOT align --- name the divergence.\\
\hspace*{2mm}Step 4: Fill the JSON match/gap fields based on Step 2/3 reasoning.\\
\\
\textbf{Rules:}\\
\hspace*{2mm}- If a dimension has no meaningful match, set MATCH to ``'' (empty string).\\
\hspace*{2mm}- If a dimension matches with no important gap, set GAP to ``'' (empty string).\\
\hspace*{2mm}- Avoid hedging like ``somewhat'' or ``loosely'' --- be specific.\\
\hspace*{2mm}- Do not output the reasoning steps; only the final JSON.\\
\textcolor{pink-color}{\textless{}/reasoning\textgreater{}}\\
\textcolor{pink-color}{\textless{}output\textgreater{}}\\
Output ONLY valid JSON, no other text:\\
\texttt{\{}\\
\hspace*{2mm}\texttt{"preconditions":}\hspace*{6mm}\texttt{\{"match": "...", "gap": "..."\},}\\
\hspace*{2mm}\texttt{"causal\_chain":}\hspace*{8mm}\texttt{\{"match": "...", "gap": "..."\},}\\
\hspace*{2mm}\texttt{"mechanism":}\hspace*{14mm}\texttt{\{"match": "...", "gap": "..."\},}\\
\hspace*{2mm}\texttt{"outcome\_impact":}\hspace*{6mm}\texttt{\{"match": "...", "gap": "..."\},}\\
\hspace*{2mm}\texttt{"structural\_pattern":} \texttt{\{"match": "...", "gap": "..."\}}\\
\texttt{\}}\\
\textcolor{pink-color}{\textless{}/output\textgreater{}}
}
\end{tcolorbox}

\begin{tcolorbox}[title=Aggregate Decision prompt (REJECT rules + ranking), breakable]
\small
\setlength{\tabcolsep}{0.1mm}{
\textcolor{pink-color}{\textless{}task\textgreater{}}\\
You have evaluated \{n\_cands\} candidate analogies for the input event below. The per-candidate match/gap summaries are provided.\\
\textcolor{pink-color}{\textless{}/task\textgreater{}}\\
\textcolor{pink-color}{\textless{}input-event\textgreater{}}\\
Input Event: \{input\_name\}\\
\hspace*{2mm}Preconditions: \{input\_precond\}\\
\hspace*{2mm}Causal Chain: \{input\_chain\}\\
\hspace*{2mm}Mechanism: \{input\_mechanism\}\\
\hspace*{2mm}Outcome \& Impact: \{input\_outcome\}\\
\hspace*{2mm}Structural Pattern: \{input\_pattern\}\\
\\
Per-candidate match/gap summary (5 dimensions each):\\
\{candidate\_summaries\}\\
\textcolor{pink-color}{\textless{}/input-event\textgreater{}}\\
\textcolor{pink-color}{\textless{}reject-rules\textgreater{}}\\
\textbf{HARD REJECTION RULES} --- apply BEFORE picking closest. A candidate that violates any of these is DISQUALIFIED, even if its chain match looks tight. All examples below are TEMPLATED (angle brackets denote slot types).\\
\\
\textbf{REJECT-1: ALIAS} --- same event under different name.\\
\hspace*{2mm}Patterns:\\
\hspace*{4mm}-- $\langle$event design / planning / pilot phase$\rangle$ disqualified for input that IS that design phase.\\
\hspace*{4mm}-- $\langle$ceremony's broadcast or telecast$\rangle$ disqualified for input that IS the ceremony.\\
\hspace*{4mm}-- $\langle$war named by calendar A or duration$\rangle$ disqualified for $\langle$same war named by calendar B$\rangle$.\\
\hspace*{4mm}-- $\langle$truncated form of input name$\rangle$ disqualified for $\langle$full input name$\rangle$ when both refer to one specific incident.\\
\\
\textbf{REJECT-2: SUB-EVENT / SUB-PERIOD of the input.}\\
\hspace*{2mm}Patterns:\\
\hspace*{4mm}-- $\langle$one named battle$\rangle$ disqualified for input ``$\langle$the multi-year war containing it$\rangle$''.\\
\hspace*{4mm}-- $\langle$key episode$\rangle$ disqualified for input ``$\langle$the decade-long movement containing it$\rangle$''.\\
\hspace*{4mm}-- $\langle$a stage / operation / sub-campaign$\rangle$ disqualified for input ``$\langle$the larger campaign that contains it$\rangle$''.\\
\hspace*{4mm}-- $\langle$named opening or final phase$\rangle$ disqualified for input ``$\langle$the encompassing event spanning that phase$\rangle$''.\\
\\
\textbf{REJECT-3: DIRECT SUPER-SET} --- input is a sub-event of the candidate.\\
\hspace*{2mm}Patterns:\\
\hspace*{4mm}-- $\langle$broader movement or campaign$\rangle$ disqualified for input ``$\langle$one named action within that movement$\rangle$''.\\
\hspace*{4mm}-- $\langle$encompassing war or era$\rangle$ disqualified for input ``$\langle$one operation / single dated incident within it$\rangle$''.\\
\\
\textbf{REJECT-4: GENERIC DESCRIPTOR of the same event type.}\\
\hspace*{2mm}Patterns:\\
\hspace*{4mm}-- $\langle$category label$\rangle$ (``Tournament'', ``Election'', ``Disaster'', ``Coup'', ``Industrial accident'', ``Ceremony'') disqualified for input ``$\langle$a specific dated/named instance$\rangle$''.\\
\hspace*{4mm}-- $\langle$unspecified-year placeholder$\rangle$ (``Historical X'', ``Old X'', ``Past X'') disqualified for input ``$\langle$a specific named X with a date$\rangle$''.\\
\\
\textbf{REJECT-5: SCOPE-CLASS MISMATCH} --- candidate's temporal/magnitude class is at least one order of magnitude different from the input on a scale-dependent axis.\\
\hspace*{2mm}Patterns (DISQUALIFY when):\\
\hspace*{4mm}-- input is a $\langle$century-long era / multi-decade period$\rangle$ AND candidate is a $\langle$single-day / single-week named incident$\rangle$.\\
\hspace*{4mm}-- input is a $\langle$multi-year movement / campaign$\rangle$ AND candidate is a $\langle$one-day protest / single dated action$\rangle$.\\
\hspace*{4mm}-- input has impact magnitude $\geq$1 OoM different from candidate on the relevant axis (casualty count, economic impact, geographic spread, role cardinality), AND the matched mechanism is scale-dependent.\\
\hspace*{2mm}\emph{EXCEPTION:} if the candidate is a representative episode that captures the input era's DEFINING mechanism (e.g., a single decisive battle for a war), allow it under sub-event reasoning --- but then it should also be checked under REJECT-2.\\
\\
If a candidate matches ANY rejection rule, exclude it. If after rejection ZERO candidates remain, fall back to the highest-ranked non-rejected candidate even if mechanism match is weak.\\
\textcolor{pink-color}{\textless{}/reject-rules\textgreater{}}\\
\textcolor{pink-color}{\textless{}reasoning\textgreater{}}\\
\textbf{REASONING-FIRST PROCESS} --- work through these steps explicitly:\\
\\
\textbf{Step 1:} For each of the \{n\_cands\} candidates, walk the rejection rules above (REJECT-1 through REJECT-5). Mark each as ELIGIBLE or REJECTED (with reason).\\
\textbf{Step 2:} Among ELIGIBLE candidates, rank by:\\
\hspace*{2mm}(a) \textbf{MECHANISM SPECIFICITY:} PREFER candidates whose causal\_chain match describes a SPECIFIC mechanism (concrete role + dynamic + outcome). DISPREFER candidates whose match is just a CATEGORY label (e.g., ``both attacks'', ``both bubbles'') even if ``famous canonical''.\\
\hspace*{2mm}(b) \textbf{SCALE MATCH} (within $\sim$1 order of magnitude, on these axes):\\
\hspace*{4mm}-- Temporal class: input duration class $\approx$ candidate duration class (era / campaign / event must align, modulo REJECT-5 exception).\\
\hspace*{4mm}-- Magnitude: casualty / economic / geographic-impact within $\sim$1 OoM.\\
\hspace*{4mm}-- Role cardinality: 1-on-1 / 1-vs-many / many-vs-many should match.\\
\hspace*{2mm}A candidate with specific mechanism but mismatched scale by $\geq$1 OoM ranks BELOW one with both mechanism and scale match. Tight specific mechanism + matched scale $>$ famous canonical pair.\\
\textbf{Step 3:} Pick top-ranked ELIGIBLE as closest.\\
\textbf{Step 4:} Decide \texttt{is\_sufficient} (yes if mechanism match is reasonable AND scale roughly aligns; minor gaps OK).\\
\textbf{Step 5:} If not sufficient, identify common\_gap and next\_round\_focus.\\
\textbf{Step 6:} Output JSON with reasoning field FIRST.\\
\\
Now answer six questions for the JSON output:\\
\textbf{Q1:} Pick the CLOSEST ELIGIBLE candidate (after applying REJECT rules).\\
\textbf{Q2:} Is it SUFFICIENT? (mechanism reasonably matches AND scale roughly aligns).\\
\textbf{Q3:} REASONING --- 2--3 sentences walking your rejection check + ranking + why the chosen candidate is best.\\
\textbf{Q4:} If insufficient, COMMON GAP across all candidates.\\
\textbf{Q5:} If insufficient, NEXT-ROUND ANGLE: (a) different mechanism family, (b) different scale/scope, (c) different outcome valence, (d) different role structure, (e) cross-domain shift.\\
\textbf{Q6:} List TOP-3 ELIGIBLE candidates (best first), excluding rejected ones.\\
\textcolor{pink-color}{\textless{}/reasoning\textgreater{}}\\
\textcolor{pink-color}{\textless{}output\textgreater{}}\\
Output ONLY valid JSON, with reasoning FIRST so the LLM commits to the reasoning before the picks:\\
\texttt{\{}\\
\hspace*{2mm}\texttt{"reasoning": "Step 1 rejection check: ...}\\
\hspace*{4mm}\texttt{Step 2 ranking: ... Final pick justification.",}\\
\hspace*{2mm}\texttt{"closest\_candidate": "$\langle$exact eligible name$\rangle$",}\\
\hspace*{2mm}\texttt{"top\_3": ["$\langle$best$\rangle$", "$\langle$2nd$\rangle$", "$\langle$3rd$\rangle$"],}\\
\hspace*{2mm}\texttt{"is\_sufficient": true OR false,}\\
\hspace*{2mm}\texttt{"common\_gap": "..."} (only if NOT sufficient; ``'' if sufficient),\\
\hspace*{2mm}\texttt{"next\_round\_focus": "..."} (only if NOT sufficient; ``'' if sufficient)\\
\texttt{\}}\\
\textcolor{pink-color}{\textless{}/output\textgreater{}}
}
\end{tcolorbox}

\section{Rubric-based Evaluation for Historical Analogy Generation}
\label{appdx:his_gen_eval}

We provide our new evaluation protocols for~\citet{li2025past}.

\begin{tcolorbox}[title=Extraction prompt -- 8-field structural decomposition, breakable]
\small
\setlength{\tabcolsep}{0.1mm}{
\textcolor{pink-color}{\textless{}task\textgreater{}}\\
You are an event-analysis assistant. Given an event description, extract eight structural fields capturing its analogical structure (Sourati 2024 ARN narrative elements + causal-chain decomposition).\\
\textcolor{pink-color}{\textless{}/task\textgreater{}}\\
\textcolor{pink-color}{\textless{}fields\textgreater{}}\\
\textbf{1. ACTORS} --- principal agents by structural ROLE (e.g., ``incumbent power'', ``rising challenger''), NOT proper names. 1--2 sentences.\\
\textbf{2. RELATIONSHIPS} --- how actors relate (alliance, rivalry, dependence, hierarchy). 1--2 sentences.\\
\textbf{3. ACTIONS} --- JSON array of 4--8 key actions in CAUSAL ORDER. Short phrases.\\
\textbf{4. GOALS} --- what each actor aims to achieve. 1--2 sentences.\\
\textbf{5. LOCATION} --- setting (geographical + institutional + domain). 1--2 sentences.\\
\textbf{6. THEME} --- highest-order causal pattern in one phrase (e.g., ``imperial overreach leads to quagmire'').\\
\textbf{7. CAUSAL\_CHAIN} --- JSON array of ordered steps with VALENCE. Each step: \texttt{\{"step": "$\langle$short phrase$\rangle$", "valence": -1|0|1\}} where $+1$ = building/positive trajectory, $-1$ = declining/negative, $0$ = neutral. Use STRICTLY $\{-1, 0, 1\}$ --- no $\pm 2$.\\
\textbf{8. CAUSAL\_ELEMENTS} --- JSON array of 5--8 specific causal factors, snake\_case short phrases. ONLY factors that play a CAUSAL role (drove or sustained the dynamic). NOT descriptions, dates, or proper names.\\
\hspace*{2mm}GOOD: \texttt{["imperial\_overextension", "asymmetric\_warfare", "domestic\_antiwar\_pressure"]}\\
\hspace*{2mm}BAD:\hspace*{2mm}\texttt{["nineteen sixty four", "Saigon", "war"]}\\
\textcolor{pink-color}{\textless{}/fields\textgreater{}}\\
\textcolor{pink-color}{\textless{}input\textgreater{}}\\
Event:\\
\{name\}: \{intro\}\\
\textcolor{pink-color}{\textless{}/input\textgreater{}}\\
\textcolor{pink-color}{\textless{}output\textgreater{}}\\
Output ONLY JSON, no other text:\\
\texttt{\{ "actors": "...", "relationships": "...", "actions": [...],}\\
\hspace*{2mm}\texttt{"goals": "...", "location": "...", "theme": "...",}\\
\hspace*{2mm}\texttt{"causal\_chain": [\{"step": "...", "valence": 1\}, ...],}\\
\hspace*{2mm}\texttt{"causal\_elements": [...] \}}\\
\textcolor{pink-color}{\textless{}/output\textgreater{}}
}
\end{tcolorbox}

\begin{tcolorbox}[title=Structural battery prompt -- chain isomorphism + direction + mapping (Items 1--3), breakable]
\small
\setlength{\tabcolsep}{0.1mm}{
\textcolor{pink-color}{\textless{}task\textgreater{}}\\
You will rate STRUCTURAL alignment between two events on three sub-items, each 0--3. Focus on STRUCTURE --- ignore proper names, dates, and domain.\\
\\
EVENT A: actions \{a\_actions\}; causal\_chain \{a\_chain\}; actors \{a\_actors\}; relationships \{a\_rel\}\\
EVENT B: actions \{b\_actions\}; causal\_chain \{b\_chain\}; actors \{b\_actors\}; relationships \{b\_rel\}\\
\textcolor{pink-color}{\textless{}/task\textgreater{}}\\
\textcolor{pink-color}{\textless{}item-1: chain isomorphism\textgreater{}}\\
Do A's action steps have functional counterparts in B, in roughly the same order?\\
\hspace*{2mm}0 = no correspondence (e.g., Tulip mania chain: speculation $\rightarrow$ bubble $\rightarrow$ crash vs Sermon on the Mount: gathering $\rightarrow$ teaching $\rightarrow$ dispersal).\\
\hspace*{2mm}1 = only $\leq$2 steps loosely correspond, rest unrelated (e.g., Russian Revolution vs French Revolution: both end in radical takeover but early steps don't map).\\
\hspace*{2mm}2 = most steps correspond with some misalignment or extras (e.g., Iraq War $\leftrightarrow$ Vietnam War: invasion $\rightarrow$ insurgency $\rightarrow$ attrition $\rightarrow$ withdrawal align, but Iraq lacks Vietnam's ``domino theory'' framing step).\\
\hspace*{2mm}3 = clean step-by-step mapping (e.g., Iraq War $\leftrightarrow$ Soviet-Afghan War: each step has a direct counterpart at similar position).\\
\textcolor{pink-color}{\textless{}/item-1\textgreater{}}\\
\textcolor{pink-color}{\textless{}item-2: direction consistency\textgreater{}}\\
Does the trajectory shape match? Chains may have DIFFERENT NUMBERS OF STEPS or different granularity --- compare overall trajectory SHAPE (where it rises/falls/turns), NOT exact step alignment.\\
\\
\textit{Examples accommodating chain-length differences:}\\
\hspace*{2mm}A: $[+1, +1, -1]$ (rise then fall, peak at 67\%); B: $[+1, +1, +1, -1, -1]$ (rise then fall, peak at 60\%) $\rightarrow$ \textbf{3} (isomorphic shape, granularity OK)\\
\hspace*{2mm}A: $[+1, +1, -1, -1]$ (rise-fall); B: $[-1, -1, +1, +1]$ (fall-rise) $\rightarrow$ \textbf{0} (opposite direction)\\
\hspace*{2mm}A: $[+1, +1, +1]$ (monotonic rise); B: $[-1, -1, -1]$ (monotonic fall) $\rightarrow$ \textbf{0}\\
\hspace*{2mm}A: $[+1, +1, +1, +1]$; B: $[+1, +1, +1]$ $\rightarrow$ \textbf{3} (same monotonic, different length OK)\\
\hspace*{2mm}A: $[+1, +1, +1]$ (monotonic); B: $[+1, -1, +1, -1]$ (oscillation) $\rightarrow$ \textbf{1}\\
\hspace*{2mm}A: $[+1, +1, -1, -1]$ (rise-fall); B: $[+1, -1, -1, +1, -1]$ (turbulent) $\rightarrow$ \textbf{1}\\
\hspace*{2mm}A: $[+1, +1, -1, -1]$ (rise to $\sim$50\%, fall); B: $[+1, +1, +1, -1]$ (rise to $\sim$75\%, fall) $\rightarrow$ \textbf{2}\\
\\
\textit{Score:}\\
\hspace*{2mm}0 = opposite trajectory direction.\hspace*{2mm}1 = different shape (monotonic vs oscillatory; rise-fall vs turbulent).\\
\hspace*{2mm}2 = roughly similar shape with some divergence in turning-point position.\hspace*{2mm}3 = isomorphic trajectory.\\
\textcolor{pink-color}{\textless{}/item-2\textgreater{}}\\
\textcolor{pink-color}{\textless{}item-3: mapping consistency\textgreater{}}\\
Can A's principal actors be mapped one-to-one to B's actors with parallel relationships? (Holyoak \& Thagard 1989)\\
\hspace*{2mm}0 = roles cannot be mapped at all (e.g., Tulip mania merchants/speculators/buyers vs Battle of Hastings Normans/Saxons/king).\\
\hspace*{2mm}1 = ambiguous mapping (one A-role maps to multiple B-roles, or relationships cross).\\
\hspace*{2mm}2 = mapping works but at least one relationship diverges (e.g., A's ``allies-turned-rivals'' maps to B's ``always rivals'').\\
\hspace*{2mm}3 = clean one-to-one with parallel relationships throughout (e.g., Iraq War $\leftrightarrow$ Vietnam War: occupier $\leftrightarrow$ occupier, insurgents $\leftrightarrow$ insurgents, host government $\leftrightarrow$ host government).\\
\textcolor{pink-color}{\textless{}/item-3\textgreater{}}\\
\textcolor{pink-color}{\textless{}reasoning\textgreater{}}\\
\textbf{REASONING-FIRST PROCESS} --- work through these steps before assigning scores:\\
\hspace*{2mm}Step 1: Walk through Item 1 (chain). Map A's steps to B's. Decide level.\\
\hspace*{2mm}Step 2: Walk through Item 2 (direction). Compare valence trajectories. Decide level.\\
\hspace*{2mm}Step 3: Walk through Item 3 (mapping). Try the actor mapping table; check parallel relationships. Decide level.\\
\hspace*{2mm}Step 4: Output JSON. The ``rationale'' field MUST come FIRST (one sentence summarizing all three judgments). Numeric scores reflect the reasoning, not the other way around.\\
\textcolor{pink-color}{\textless{}/reasoning\textgreater{}}\\
\textcolor{pink-color}{\textless{}output\textgreater{}}\\
\texttt{\{ "rationale": "one sentence summarizing the three judgments and the dominant reason behind each score",}\\
\hspace*{2mm}\texttt{"chain\_isomorphism": 0, "direction\_consistency": 0, "mapping\_consistency": 0 \}}\\
\textcolor{pink-color}{\textless{}/output\textgreater{}}
}
\end{tcolorbox}

\begin{tcolorbox}[title=Pragmatic battery prompt -- overlap sufficiency + idiosyncrasy coverage (Items 4 and 6), breakable]
\small
\setlength{\tabcolsep}{0.1mm}{
\textcolor{pink-color}{\textless{}task\textgreater{}}\\
You will judge whether two events share IDIOSYNCRATIC causal elements vs only GENERIC ones.\\
\textcolor{pink-color}{\textless{}/task\textgreater{}}\\
\textcolor{pink-color}{\textless{}tagging-rule\textgreater{}}\\
\textbf{IDIOSYNCRATIC (I)} = element names a SPECIFIC MECHANISM (specifies HOW or WHY). Short phrases like ``imperial\_overextension'' count, because they specify a particular mechanism --- not just a generic category.\\
\\
\textbf{GENERIC (G)} = element is a CATEGORY LABEL (specifies WHAT but not HOW or WHY). ``war'', ``election'', ``regime\_change'' apply to thousands of events without distinguishing mechanism.\\
\\
\textbf{Decision rule:} Could this element distinguish one historical event from a generic example of its category? If YES $\rightarrow$ I. If it would apply to nearly any event of that type $\rightarrow$ G.\\
\\
\textit{IDIOSYNCRATIC examples (specific mechanism --- short phrasing OK):}\\
\hspace*{2mm}\texttt{imperial\_overextension}, \texttt{asymmetric\_warfare}, \texttt{speculative\_bubble},\\
\hspace*{2mm}\texttt{herd\_valuation\_beyond\_fundamentals}, \texttt{ally\_defection\_under\_attritional\_pressure},\\
\hspace*{2mm}\texttt{logistical\_overstretch\_in\_hostile\_terrain}, \texttt{collective\_boycott\_as\_diplomatic\_weapon},\\
\hspace*{2mm}\texttt{nonviolent\_mass\_mobilization\_against\_single\_party\_rule},\\
\hspace*{2mm}\texttt{preemptive\_strike\_to\_disrupt\_alliance}, \texttt{domestic\_antiwar\_pressure\_forcing\_withdrawal},\\
\hspace*{2mm}\texttt{Old\_Regime\_dual\_power\_radicalization}.\\
\\
\textit{GENERIC examples (bare category --- almost any event has these):}\\
\hspace*{2mm}\texttt{war}, \texttt{election}, \texttt{leadership\_change}, \texttt{economic\_decline}, \texttt{treaty\_signing},\\
\hspace*{2mm}\texttt{international\_pressure}, \texttt{military\_defeat}, \texttt{domestic\_protest}, \texttt{regime\_collapse},\\
\hspace*{2mm}\texttt{religious\_conflict}, \texttt{civil\_unrest}, \texttt{diplomatic\_crisis}.\\
\\
\textit{Borderline:}\\
\hspace*{2mm}\texttt{speculative\_bubble} $\rightarrow$ I (specifies herd-valuation mechanism); \texttt{market\_crash} $\rightarrow$ G (just outcome category).\\
\hspace*{2mm}\texttt{imperial\_overextension} $\rightarrow$ I (specifies overextension as cause); \texttt{imperial\_decline} $\rightarrow$ G (outcome only).\\
\textcolor{pink-color}{\textless{}/tagging-rule\textgreater{}}\\
\textcolor{pink-color}{\textless{}input\textgreater{}}\\
EVENT A theme: \{a\_theme\}; causal\_elements: \{a\_elements\}\\
EVENT B theme: \{b\_theme\}; causal\_elements: \{b\_elements\}\\
\textcolor{pink-color}{\textless{}/input\textgreater{}}\\
\textcolor{pink-color}{\textless{}reasoning\textgreater{}}\\
1. List shared elements that A and B genuinely share IN CONCEPT (identical wording NOT required --- ``imperial\_overextension'' and ``overstretch\_of\_great\_power'' count as the same).\\
2. For each shared element: SPECIFIC MECHANISM (HOW/WHY) $\rightarrow$ I, or BARE CATEGORY (WHAT only) $\rightarrow$ G.\\
3. Score the two items below.\\
\textcolor{pink-color}{\textless{}/reasoning\textgreater{}}\\
\textcolor{pink-color}{\textless{}item-4: overlap sufficiency\textgreater{}}\\
Are there enough IDIOSYNCRATIC shared elements?\\
\hspace*{2mm}0 = no shared elements at all.\\
\hspace*{2mm}1 = only GENERIC shared elements (just bare categories).\\
\hspace*{2mm}2 = exactly 1 idiosyncratic shared element.\\
\hspace*{2mm}3 = $\geq$2 idiosyncratic shared elements.\\
\textcolor{pink-color}{\textless{}/item-4\textgreater{}}\\
\textcolor{pink-color}{\textless{}item-6: idiosyncrasy coverage\textgreater{}}\\
Of A's distinctive (specific-mechanism) causal elements, how many find a counterpart in B?\\
\hspace*{2mm}0 = none.\hspace*{2mm}1 = $\sim$1/4.\hspace*{2mm}2 = $\sim$1/2.\hspace*{2mm}3 = most ($\geq$3/4).\\
\textcolor{pink-color}{\textless{}/item-6\textgreater{}}\\
\textcolor{pink-color}{\textless{}output\textgreater{}}\\
\texttt{\{ "shared\_elements": [\{"element": "...", "type": "I", "why": "..."\}, ...],}\\
\hspace*{2mm}\texttt{"overlap\_sufficiency": 0, "idiosyncrasy\_coverage": 0,}\\
\hspace*{2mm}\texttt{"rationale": "one short sentence" \}}\\
\textcolor{pink-color}{\textless{}/output\textgreater{}}
}
\end{tcolorbox}

\begin{tcolorbox}[title=Surface similarity prompt -- attribute overlap (Item 5), breakable]
\small
\setlength{\tabcolsep}{0.1mm}{
\textcolor{pink-color}{\textless{}task\textgreater{}}\\
Score SURFACE similarity (concrete attribute overlap) between two events on a 0--3 scale. Surface = entity-level / attribute-level overlap (era, geography, named entities, domain). NOT structural pattern matching.\\
\\
EVENT A: \{a\_name\}; actors \{a\_actors\}; location \{a\_loc\}; goals \{a\_goals\}\\
EVENT B: \{b\_name\}; actors \{b\_actors\}; location \{b\_loc\}; goals \{b\_goals\}\\
\textcolor{pink-color}{\textless{}/task\textgreater{}}\\
\textcolor{pink-color}{\textless{}categories\textgreater{}}\\
Each of the 5 categories below counts AT MOST ONCE --- do not double-count nested matches.\\
\\
\textbf{ERA} --- same time period within $\sim$30 years (count once even if same century AND decade AND year). E.g., ``both early 20th C'' = ``both 1960s'' = 1 attribute.\\
\\
\textbf{GEOGRAPHY} --- same continent OR country OR city (count once at most specific shared level). E.g., ``both Europe'' = ``both France'' = ``both Paris'' = 1 attribute.\\
\\
\textbf{NAMED ENTITIES} --- at least one specific named person/organization/country in BOTH. Count ONCE even if multiple overlap (``both involve US AND Britain'' still = 1).\\
\\
\textbf{DOMAIN} --- same broad activity TYPE; pick MOST SPECIFIC SHARED level only, do NOT add nested layers. E.g., ``both war'' + ``both naval war'' + ``both WWI naval war'' = 1 attribute (the most specific shared = ``WWI naval war'').\\
\\
\textbf{OPERATION TYPE} --- specific operation (olympic boycott, speculative bubble, hostage standoff, ambush, scientific R\&D program). ONLY count if different from / more specific than DOMAIN.\\
\textcolor{pink-color}{\textless{}/categories\textgreater{}}\\
\textcolor{pink-color}{\textless{}scoring\textgreater{}}\\
Count how many of the 5 categories matched (each 0 or 1):\\
\hspace*{2mm}0 = 0 categories shared (e.g., 1980 Moscow Olympics $\leftrightarrow$ Norman Conquest of England).\\
\hspace*{2mm}1 = exactly 1 category shared, typically only DOMAIN (e.g., 2008 GFC $\leftrightarrow$ South Sea Bubble: only DOMAIN = financial bubble).\\
\hspace*{2mm}2 = 2 categories shared (e.g., Iraq War 2003 $\leftrightarrow$ Vietnam War: DOMAIN = war + NAMED ENTITIES = US).\\
\hspace*{2mm}3 = $\geq$3 categories shared, OR sub-instance / near-self.\\
\hspace*{4mm}E.g., Iraq War 2003 $\leftrightarrow$ Iraq War 2003 invasion phase (sub-instance --- automatic 3).\\
\hspace*{4mm}E.g., Indo-Pakistani War 1965 $\leftrightarrow$ Indo-Pakistani War 1971 (DOMAIN + GEOGRAPHY + NAMED ENTITIES + ERA = 4 $\rightarrow$ 3).\\
\hspace*{4mm}E.g., Yalta $\leftrightarrow$ Potsdam Conference (DOMAIN + GEOGRAPHY + NAMED ENTITIES + ERA = 4 $\rightarrow$ 3).\\
\\
\textbf{CRITICAL:} Be strict about anti-double-count. If two events are both ``20th-century state-coercion famines in Eurasia'', that's 3 categories (era + domain + geography), NOT 5.\\
\textcolor{pink-color}{\textless{}/scoring\textgreater{}}\\
\textcolor{pink-color}{\textless{}output\textgreater{}}\\
\texttt{\{ "shared\_attributes": [...], "surface\_similarity": 0, "rationale": "one short sentence" \}}\\
\textcolor{pink-color}{\textless{}/output\textgreater{}}
}
\end{tcolorbox}

\begin{tcolorbox}[title=System alignment prompt -- canonical-form abstraction with reference levels (Item 8), breakable]
\small
\setlength{\tabcolsep}{0.1mm}{
\textcolor{pink-color}{\textless{}task\textgreater{}}\\
You will rate the structural alignment of two events' causal patterns on a 0--3 scale. You are given each event's actions, goals, and theme.\\
\\
EVENT A: actions \{a\_actions\}; goals \{a\_goals\}; theme \{a\_theme\}\\
EVENT B: actions \{b\_actions\}; goals \{b\_goals\}; theme \{b\_theme\}\\
\textcolor{pink-color}{\textless{}/task\textgreater{}}\\
\textcolor{pink-color}{\textless{}step-1: canonical form\textgreater{}}\\
Restate each event's CAUSAL PATTERN as a one-clause statement using ABSTRACT roles only (strip proper names, dates, domains). Use the schema:\\
\hspace*{2mm}``$\langle$role X$\rangle$ $\langle$action$\rangle$ $\langle$role Y$\rangle$, resulting in $\langle$outcome$\rangle$, because of $\langle$mechanism$\rangle$''.\\
\textcolor{pink-color}{\textless{}/step-1\textgreater{}}\\
\textcolor{pink-color}{\textless{}step-2: rate by reference pair\textgreater{}}\\
\textbf{LEVEL 0} --- Different mechanism families. Almost no structural overlap.\\
\hspace*{2mm}\textit{Reference 0a:}\\
\hspace*{4mm}A: ``speculators inflate asset price, then panic sell when confidence breaks, collapsing market because herd behavior amplified valuation beyond fundamentals''.\\
\hspace*{4mm}B: ``religious leader challenges established doctrine, splits the community into rival factions, producing schism because ideological differences cannot be reconciled''.\\
\hspace*{4mm}$\rightarrow$ ``speculation bubble'' vs ``religious schism'' share NO mechanism. Different families.\\
\\
\textbf{LEVEL 1} --- Same general PROCESS CATEGORY, but specific mechanism differs.\\
\hspace*{2mm}\textit{Reference 1a:}\\
\hspace*{4mm}A: ``expanding power launches sustained military offensive across multiple fronts, gets bogged down in attrition warfare, suffers strategic loss because coalition opponents organize sustained counter-resistance''.\\
\hspace*{4mm}B: ``ambitious power launches decisive surprise offensive, achieves rapid initial gains then collapses in attritional defense, suffers strategic loss because logistical overstretch overwhelms operational depth''.\\
\hspace*{4mm}$\rightarrow$ Both ``military offensive that fails through attrition'', but A's mechanism = COALITION COUNTER-RESISTANCE while B's = LOGISTICAL OVERSTRETCH.\\
\\
\hspace*{2mm}\textit{Reference 1b:}\\
\hspace*{4mm}A: ``religious community participates in service-formation rite during commemorative period, reinforcing communal identity through shared ritual roles''.\\
\hspace*{4mm}B: ``religious community performs cyclical-renewal procession during commemorative period, reinforcing communal identity through symbolic re-enactment''.\\
\hspace*{4mm}$\rightarrow$ Both ``religious commemorative ritual'', but A's mechanism = SERVICE FORMATION while B's = SYMBOLIC RE-ENACTMENT.\\
\\
\textbf{LEVEL 2} --- Same mechanism, but with noticeable differences in scope, role granularity, or outcome scale.\\
\hspace*{2mm}\textit{Reference 2a:}\\
\hspace*{4mm}A: ``expanding maritime empire launches piecemeal coercive intervention against fragmented island rulers, escalates from blockade to armed annexation, resulting in colonial control over a regional archipelago because local rulers fail to unite''.\\
\hspace*{4mm}B: ``expanding continental empire launches sustained coercive intervention against fragmented tribal confederations, escalates from punitive raids to full territorial conquest, resulting in colonial control over a continental landmass because local powers fail to coalesce''.\\
\hspace*{4mm}$\rightarrow$ Same mechanism (``piecemeal colonial conquest exploiting fragmentation''), different scope (Dutch in Lombok vs French in Algeria).\\
\\
\hspace*{2mm}\textit{Reference 2b:}\\
\hspace*{4mm}A: ``speculators bid up prices of a novel agricultural commodity beyond use-value, market collapses in panic when confidence breaks, producing localized merchant losses because herd valuation exceeded what end-buyers would pay''.\\
\hspace*{4mm}B: ``speculators bid up prices of a government-backed corporate share beyond intrinsic value, market collapses in panic when confidence breaks, producing widespread elite financial losses and political fallout because herd valuation exceeded the company's actual revenue''.\\
\hspace*{4mm}$\rightarrow$ Same mechanism, different scope (tulips vs South Sea).\\
\\
\textbf{LEVEL 3} --- Near-identical canonical forms: same role asymmetries, mechanism, outcome trajectory.\\
\hspace*{2mm}\textit{Reference 3a:}\\
\hspace*{4mm}A: ``weaker insurgent force conducts surprise attack on isolated detachment of stronger occupier, achieves limited tactical victory, converts to political momentum because symbolic win galvanizes broader coalition support against the occupier''.\\
\hspace*{4mm}B: ``weaker defender lures isolated detachment of stronger advancing army into ambush, achieves limited tactical victory, converts to political momentum because the symbolic win galvanizes broader resistance against the advancing power''.\\
\hspace*{4mm}$\rightarrow$ Near-identical structure (Battle of Princeton vs Battle of Cowpens).\\
\\
\hspace*{2mm}\textit{Reference 3b:}\\
\hspace*{4mm}A: ``incumbent single-party authoritarian regime faces sustained nonviolent mass mobilization, security apparatus defects under public pressure, regime collapses without prolonged civil war because coercive capacity disintegrates''.\\
\hspace*{4mm}B: ``incumbent single-party authoritarian regime faces sustained nonviolent mass mobilization, security apparatus defects under public pressure, regime collapses with brief armed clash because coercive capacity disintegrates''.\\
\hspace*{4mm}$\rightarrow$ Near-identical with only minor difference (East Germany 1989 vs Romania 1989).\\
\textcolor{pink-color}{\textless{}/step-2\textgreater{}}\\
\textcolor{pink-color}{\textless{}output\textgreater{}}\\
\texttt{\{ "a\_canonical": "$\langle$role X$\rangle$ ... resulting in ... because ...",}\\
\hspace*{2mm}\texttt{"b\_canonical": "$\langle$role X$\rangle$ ... resulting in ... because ...",}\\
\hspace*{2mm}\texttt{"system\_alignment": 0,}\\
\hspace*{2mm}\texttt{"rationale": "one short sentence explaining which reference level your A/B pair most resembles" \}}\\
\textcolor{pink-color}{\textless{}/output\textgreater{}}
}
\end{tcolorbox}

\begin{tcolorbox}[title=Novelty prompt -- canonical vs novel pair detection (Item 9), breakable]
\small
\setlength{\tabcolsep}{0.1mm}{
\textcolor{pink-color}{\textless{}task\textgreater{}}\\
Judge whether the analogy pair (A, B) is NOVEL or CANONICAL on a 0--3 scale.\\
\\
\textbf{Definitions:}\\
\hspace*{2mm}\textbf{CANONICAL} = pair appears frequently in history textbooks, classroom materials, news columns, or popular media (e.g., ``Iraq War vs Vietnam''; ``2008 GFC vs 1929 Depression''; ``Russian Revolution vs French Revolution'').\\
\hspace*{2mm}\textbf{NOVEL} = a non-obvious pair requiring structural insight to draw (cross-domain, cross-era, or cross-scale).\\
\\
EVENT A: \{a\_name\}; EVENT B: \{b\_name\}\\
\textcolor{pink-color}{\textless{}/task\textgreater{}}\\
\textcolor{pink-color}{\textless{}reasoning\textgreater{}}\\
\textit{Step 1:} Briefly assess: would a college history textbook or popular-press history article likely treat (A, B) as a canonical comparison?\\
\textit{Step 2:} Score on this scale:\\
\\
\hspace*{2mm}\textbf{0} = canonical pair, repeatedly drawn in standard textbooks.\\
\hspace*{4mm}E.g., 2008 GFC $\leftrightarrow$ 1929 Great Depression; Iraq War 2003 $\leftrightarrow$ Vietnam War; Russian Revolution $\leftrightarrow$ French Revolution; WW2 $\leftrightarrow$ WW1.\\
\\
\hspace*{2mm}\textbf{1} = well-known pair within history community, frequently drawn.\\
\hspace*{4mm}E.g., Iraq War 2003 $\leftrightarrow$ Soviet-Afghan War; Pearl Harbor $\leftrightarrow$ 9/11.\\
\\
\hspace*{2mm}\textbf{2} = less commonly drawn pair, requires some specialist knowledge.\\
\hspace*{4mm}E.g., Suez Crisis $\leftrightarrow$ Falklands War; 1980 Olympics boycott $\leftrightarrow$ 1976 Montreal African boycott.\\
\\
\hspace*{2mm}\textbf{3} = novel cross-domain or cross-scale pair requiring structural insight.\\
\hspace*{4mm}E.g., 2008 GFC $\leftrightarrow$ Chernobyl Disaster; Apollo Program $\leftrightarrow$ Manhattan Project (both: emergency cross-disciplinary R\&D under deadline); French Revolution $\leftrightarrow$ Iranian Revolution 1979.\\
\textcolor{pink-color}{\textless{}/reasoning\textgreater{}}\\
\textcolor{pink-color}{\textless{}output\textgreater{}}\\
\texttt{\{ "is\_canonical": true, "novelty": 0, "rationale": "one short sentence" \}}\\
\textcolor{pink-color}{\textless{}/output\textgreater{}}
}
\end{tcolorbox}

\begin{tcolorbox}[title=Self-analogy 3-class prompt -- IDENTITY / SIBLING\_TRIVIAL / DISTINCT, breakable]
\small
\setlength{\tabcolsep}{0.1mm}{
\textcolor{pink-color}{\textless{}task\textgreater{}}\\
Classify the relationship between historical events ``\{input\_name\}'' and ``\{candidate\_name\}''. Return ONE of three classes: \textbf{IDENTITY} / \textbf{SIBLING\_TRIVIAL} / \textbf{DISTINCT}.\\
\textcolor{pink-color}{\textless{}/task\textgreater{}}\\
\textcolor{pink-color}{\textless{}class IDENTITY\textgreater{}}\\
They refer to the SAME real-world event/thing, OR one is a strict part of the other. The candidate is NOT a valid analogy because it IS or is WITHIN the input.\\
\\
\textbf{(a) ALIAS} --- different names for the same specific event:\\
\hspace*{2mm}-- $\langle$event under codename$\rangle$ $\equiv$ $\langle$same event under descriptive popular name$\rangle$\\
\hspace*{2mm}-- $\langle$event under regional/local name$\rangle$ $\equiv$ $\langle$same event under translated/global name$\rangle$\\
\hspace*{2mm}-- $\langle$event under one academic framing$\rangle$ $\equiv$ $\langle$same event under different academic/scholarly framing$\rangle$ (e.g., ``$\langle$Movement$\rangle$ of YYYY'' $\equiv$ ``$\langle$Same period$\rangle$ Collapse of $\langle$Container$\rangle$'')\\
\hspace*{2mm}-- $\langle$event under truncated form$\rangle$ $\equiv$ $\langle$event under full form$\rangle$\\
\hspace*{2mm}-- $\langle$event with date qualifier$\rangle$ $\equiv$ $\langle$same event with no date$\rangle$\\
\\
\textbf{(b) CATEGORY-SYNONYM} --- both names refer to the same category in standard usage (near-synonymous terms used interchangeably; same concept under formal vs. informal name).\\
\\
\textbf{(c) CONCEPT-LEVEL ALIAS WITH NO TOKEN OVERLAP} --- different names with no shared word but referring to the same historical moment or transformation (e.g., regional name $\leftrightarrow$ international name; political label $\leftrightarrow$ cultural label for the same period).\\
\\
\textbf{(d) CONTAINMENT} --- the candidate is a phase, sub-period, sub-event, sub-action, named operation, or specific incident WITHIN the input event's timeline/agenda/scope (or vice versa):\\
\hspace*{2mm}-- Sub-event during input: candidate happened DURING input event's timeline.\\
\hspace*{2mm}-- Member of movement: candidate is a specific named action within the input's broader movement/campaign.\\
\hspace*{2mm}-- Phase/operation: candidate is a named operation within input's larger campaign.\\
\textcolor{pink-color}{\textless{}/class IDENTITY\textgreater{}}\\
\textcolor{pink-color}{\textless{}class SIBLING\_TRIVIAL\textgreater{}}\\
Distinct events, but they share a sibling-membership in one institution/series/actor's campaign such that their similarity is structurally trivial (explained by common cause, not by analogy):\\
\hspace*{2mm}-- Numbered/recurring institutional series: both are members of an annual ceremony, sequential conference, world exposition, championship cycle, election cycle, etc.\\
\hspace*{2mm}-- Same actor + same era + same action-type: same primary actor doing the same kind of action against similar targets in the same era (parallel campaigns of one program).\\
\hspace*{2mm}-- Same series of meetings/treaties: both are installments of the same diplomatic/political series.\\
\\
The events are independent, but their similarity comes from common cause/institution rather than from genuine structural analogy.\\
\textcolor{pink-color}{\textless{}/class SIBLING\_TRIVIAL\textgreater{}}\\
\textcolor{pink-color}{\textless{}class DISTINCT\textgreater{}}\\
Genuinely different events that deserve to be judged on their structural merits:\\
\hspace*{2mm}-- Different events sharing only vocabulary (e.g., two wars whose names share a word but are distinct conflicts).\\
\hspace*{2mm}-- Different events in the same domain at different times/places, NOT members of the same series.\\
\hspace*{2mm}-- Cross-domain analogs (e.g., a war and a corporate takeover).\\
\textcolor{pink-color}{\textless{}/class DISTINCT\textgreater{}}\\
\textcolor{pink-color}{\textless{}reasoning\textgreater{}}\\
Reason step by step BEFORE giving the final answer:\\
1. Identify the specific historical event each name refers to. Are they the same real-world referent under different names?\\
2. If different referents, check: is one a strict CONTAINMENT (phase / sub-event / sub-action / member-of-movement / sub-period) of the other?\\
3. If neither identity nor containment, check: are they SIBLING members of one institutional series / one actor's parallel campaign / one ceremony cycle?\\
4. Otherwise: DISTINCT.\\
\textcolor{pink-color}{\textless{}/reasoning\textgreater{}}\\
\textcolor{pink-color}{\textless{}output\textgreater{}}\\
Reasoning: $\langle$2--3 sentences$\rangle$\\
Answer: IDENTITY \textbf{or} SIBLING\_TRIVIAL \textbf{or} DISTINCT\\
\textcolor{pink-color}{\textless{}/output\textgreater{}}
}
\end{tcolorbox}

\textbf{Seven failure modes resolved.} Concretely, in Table~\ref{tab:failure}, we present examples showing the seven systematic failure modes of descriptive matching via MDS decomposition, each resolved by \ours decomposition. It also echoes the importance of alignment via $M(E)$ as predicted by our theory.

 \subsection{Failure modes of descriptive matching}
 \label{appdx:failure}
\begin{table}[t]
\centering\small
\caption{Seven failure modes of descriptive matching resolved by \ours decomposition.}
\label{tab:failure}
\begin{tabular}{@{}lll@{}}
\toprule
\textbf{Failure Mode} & \textbf{Descriptive Match~\citep{li2025past}} & \textbf{\ours decomposition} \\
\midrule
Part-whole confusion & Arab Spring $\to$ Tunisian Rev. & $\to$ 1989 E.\ Europe \\
Same-war echo & Expedition of 1000 $\to$ Volturno & $\to$ Cuban Revolution \\
Self-analogy & Iraq War $\to$ 2003 Invasion & $\to$ Soviet-Afghan War \\
Causal direction error & Battle of Bulge $\to$ Market Garden & $\to$ Michael Offensive \\
Surface feature match & Caesar $\to$ Lincoln & $\to$ Cromwell's death \\
Proximity bias & J.\ Brown's Raid $\to$ Bleeding Kansas & $\to$ Easter Rising \\
Domain closure & Canudos $\to$ Siege of Balaler & $\to$ Satsuma Rebellion \\
\bottomrule
\end{tabular}
\end{table}